%% file: main.tex
\documentclass{article}

\usepackage{microtype}
\usepackage{graphicx}
\usepackage{subfigure}
\usepackage{booktabs} 

\usepackage{epstopdf}

\usepackage{hyperref}

\usepackage[accepted]{icml2019}

\input{mystyle}

\icmltitlerunning{Dropout and Nuclear Norm Regularization}

\begin{document}

\twocolumn[
\icmltitle{On Dropout and Nuclear Norm Regularization}



\icmlsetsymbol{equal}{*}

\begin{icmlauthorlist}
\icmlauthor{Poorya Mianjy}{csjhu}
\icmlauthor{Raman Arora}{csjhu}
\end{icmlauthorlist}

\icmlaffiliation{csjhu}{Department of Computer Science, Johns Hopkins University, Baltimore, MD, USA}

\icmlcorrespondingauthor{Raman Arora}{arora@cs.jhu.edu}

\icmlkeywords{Deep Learning, Dropout, Inductive Bias, Nuclear Norm}

\vskip 0.3in
]



\printAffiliationsAndNotice{}  

\begin{abstract}
We give a formal and complete characterization of the explicit regularizer induced by dropout in deep linear networks with squared loss. We show that (a) the explicit regularizer is composed of an $\ell_2$-path regularizer and other terms that are also re-scaling invariant, (b) the convex envelope of the induced regularizer is the squared nuclear norm of the network map, and (c) for a sufficiently large dropout rate, we characterize the global optima of the dropout objective. We validate our theoretical findings with empirical results.
\end{abstract}
\input{intro}
\input{main_results}
\input{proofs}
\input{exp}
\input{disc}

\section*{Acknowledgements}
\noindent This research was supported in part by NSF BIGDATA grant IIS-1546482. 


\bibliography{main.bib}
\bibliographystyle{icml2019}

\input{appendix}
\end{document}

%% file: mystyle.tex
\usepackage{amsmath,amssymb,amsthm}
\usepackage{enumerate}
\usepackage{multirow}
\usepackage{algorithm,algorithmic}

\usepackage{wrapfig}
\usepackage[dvipsnames]{xcolor}
\definecolor{blue}{rgb}{0.0, 0.0, 0.3}
\definecolor{red}{rgb}{0.3, 0.0, 0.0}
\definecolor{green}{rgb}{0.0, 0.3, 0.0}

\newcommand{\tr}{\operatorname{Trace}}
\newcommand{\trace}[1]{\normalfont\textrm{Tr}\left(#1\right)}
\newcommand{\diag}[1]{\operatorname{diag}\left(#1\right)}
\newcommand{\rank}[1]{\operatorname{rank}\left(#1\right)}

\newcommand{\infim}[2]{\underset{#1}{\textrm{inf}} \ #2}
\newcommand{\minim}[2]{\underset{#1}{\textrm{min}} \ #2}
\newcommand{\maxim}[2]{\underset{#1}{\textrm{max}} \ #2}

\newcommand{\x}{{\normalfont\textrm x}}
\newcommand{\y}{{\normalfont\textrm y}}
\renewcommand{\v}{{\textrm v}}
\newcommand{\e}{{\textrm e}}
\renewcommand{\u}{{\textrm u}}

\newcommand{\0}{{\textrm 0}}
\newcommand{\1}{\textrm 1}

\newcommand{\w}{{\normalfont\textrm w}}
\renewcommand{\b}{{\textrm b}}
\renewcommand{\d}{{\textrm d}}

\newcommand{\I}{\textrm{I}}
\newcommand{\A}{\textrm{A}}
\newcommand{\B}{\textrm{B}}
\newcommand{\M}{\normalfont\textrm{M}}

\newcommand{\W}{{\normalfont\textrm{W}}}
\newcommand{\C}{{\normalfont\textrm{C}}}

\newcommand{\U}{\textrm{U}}
\newcommand{\V}{\textrm{V}}
\renewcommand{\P}{\textrm{P}}

\newcommand{\X}{{\normalfont\textrm{X}}}
\newcommand{\Y}{{\normalfont\textrm{Y}}}

\newcommand{\cD}{\mathcal{D}}

\newcommand{\cE}{\mathcal{E}}
\newcommand{\cL}{\mathcal{L}}

\newcommand{\cN}{\mathcal{N}}

\newcommand{\cX}{\mathcal{X}}
\newcommand{\cY}{\mathcal{Y}}
\newcommand{\cS}{\mathcal{S}}

\newcommand{\bb}{\mathbb}
\newcommand{\R}{\bb R}

\newcommand{\E}{\bb E}
\newcommand{\bE}{\bb E}

\usepackage{tikz}
\usetikzlibrary{matrix,chains,positioning,decorations.pathreplacing,arrows}
\usetikzlibrary{positioning,calc}
\usepackage{tcolorbox} 

\newtheorem{theorem}{Theorem}[section]
\newtheorem{lemma}[theorem]{Lemma}
\newtheorem{proposition}[theorem]{Proposition}

\newtheorem{definition}[theorem]{Definition}
\newtheorem{remark}[theorem]{Remark}
\newtheorem{prop}[theorem]{Proposition}
\newtheorem{corollary}[theorem]{Corollary}

\usepackage{flushend}

%% file: intro.tex
\section{Introduction}\label{sec:intro}

Deep learning is revolutionizing the technological world with recent advances in artificial intelligence. However, a formal understanding of when or why deep learning algorithms succeed has remained elusive. Recently, a series of works focusing on computational learning theoretic aspects of deep learning have implicated inductive biases due to various algorithmic choices to be a crucial potential explanation~\cite{zhang2016understanding,gunasekar2018characterizing,neyshabur2014search,martin2018implicit,mianjy2018implicit}. Here, we examine such an implicit bias of dropout in deep linear networks. 

Dropout is a popular algorithmic approach that helps training deep neural networks that generalize better~\cite{hinton2012improving,srivastava2014dropout}. Inspired by the reproduction model in the evolution of advanced organisms, dropout training aims at breaking co-adaptation among neurons by dropping them independently and identically according to a Bernoulli random variable.

Here, we restrict ourselves to simpler networks; we consider multi-layered feedforward networks with linear activations~\cite{goodfellow2016deep}. While the overall function is linear, the representation in factored form makes the optimization landscape non-convex and hence, challenging to analyze. More importantly, we argue that the fact we choose to represent a linear map in a factored form has important implications to the learning problem, akin in many ways to the implicit bias due to stochastic optimization algorithms and various algorithmic heuristics used in deep learning~\cite{gunasekar2017implicit,li2018algorithmic,azizan2019stochastic}. 

Several recent works have investigated the optimization landscape properties of deep linear networks~\cite{baldi1989neural,saxe2013exact,kawaguchi2016deep,hardt2016identity,laurent2018deep}, as well as the implicit bias due to first-order optimization algorithms for training such networks~\cite{gunasekar2018implicit,ji2018gradient}, and the convergence rates of such algorithms~\cite{bartlett2018gradient,arora2018convergence}. The focus here is to have a similar development for dropout when training deep linear networks.

\subsection{Notation}\label{sec:notation}
For an integer $i$, $[i]$ represents the set $\{ 1,\ldots,i \}$, $\e_i$ denotes the $i$-th standard basis, and $\1_i \in \R^{i}$ is the vector of all ones. The set of all $k$-combinations of a set $\cS$ is denoted by $\cS \choose k$. We denote linear operators and vectors by Roman capital and lowercase letters, respectively, e.g. $\Y$ and $\y$. Scalar variables are denoted by lower case letters (e.g. $y$) and sets by script letters, e.g. $\cY$. We denote the $\ell_2$ norm of a vector $\x$ by $\| \x\|$. For a matrix $\X$, $\| \X \|_F$ denotes the Frobenius norm, $\| \X\|_*$ denotes the nuclear norm, and $\sigma_i(\X)$ is the $i$-th largest singular value of matrix $\X$. For $\X\in \R^{d_2\times d_1}$ and a positive definite matrix $\C\in\R^{d_1\times d_1}$, $\| \X \|_\C^2 := \trace{\X\C\X^\top}$. The standard inner product between two matrices $\X,\X'$, is denoted by $\langle \X, \X' \rangle := \trace{\X^\top \X'}$. We denote the $i$-th column and the $j$-th row of a matrix $\X$ with vectors $\x_{:i}$ and $\x_{j:}$, both in column form. The vector of diagonal elements of $\X$ is denoted as $\diag{\X}$. The diagonal matrix with diagonal entries as the elements of a vector $\x$ is denoted as $\diag{\x}$. With a slight abuse of notation, we use $\{ \W_i \}$ as a shorthand for the tuple $(\W_1, \ldots, \W_{k+1})$. 

\subsection{Problem Setup}\label{sec:pre}
We consider the hypotheses class of multilayer feed-forward linear networks with input dimension $d_0$, output dimension $d_{k+1}$, $k$ hidden layers of widths $d_1, \ldots, d_k,$ and linear transformations $\W_i \in \R^{d_i \times d_{i-1}}$, for $i = 1, \ldots, k+1:$
$$\cL_{\{d_i\}}=\{ g:\x \mapsto \W_{k+1}\cdots \W_{1}\x, \ \W_i \in \R^{d_{i} \times d_{i-1}}\}.$$ 
We refer to the set of $k+1$ integers $\{ d_i \}_{i=0}^{k+1}$ specifying the width of each layer as the \emph{architecture} of the function class, the set of the weight parameters $\{\W_i\}_{i=1}^{k+1}$ as an \emph{implementation}, or an element of the function class, and $\W_{k+1\to 1}:=\W_{k+1}\W_k\cdots\W_1$ as the \emph{network map}.

\renewcommand{\arraystretch}{1.5}
\begin{table}
\caption{Key terms, corresponding symbols, and descriptions.}
\small
\begin{center}
\begin{tabular}{| c | c | c |} \hline
{\bf Term} & {\bf Symbol} &  {\bf Description} \\ \hline
architecture & $\{ d_i \}$ &  $(d_0,\ldots,d_{k+1})$ \\ \hline
implementation & $\{ \W_i \}$ & $(\W_0,\ldots,\W_{k+1})$ \\ \hline
network map & $\W_{k+1\to 1}$ & $\W_{k+1}\W_k\cdots \W_1$ \\ \hline
population risk & $L(\{ \W_i \})$ &  $\bE_{\x,\y}[\| \y - \W_{k+1 \to 1}\x\|^2]$ \\ \hline
dropout objective & $L_\theta(\{ \W_i \})$ &  see Equation~\eqref{eq:dropout_obj} \\ \hline
explicit regularizer & $R(\{ \W_i \})$ & {$ L_\theta(\{ \W_i \})\!-\!L(\{ \W_i \})$} \\ \hline
induced regularizer & $\Theta(\M)$ &  see Equation~\eqref{eq:induced_reg} \\ \hline
\end{tabular}
\end{center}
\label{tab:notation}
\end{table}

The focus here is on \emph{deep regression} with dropout under $\ell_2$ loss, which is widely used in computer vision tasks, including human pose estimation~\cite{toshev2014deeppose}, facial landmark detection, and age estimation~\cite{lathuiliere2019comprehensive}. More formally, we study the following learning problem for deep linear networks. Let $\cX \subseteq \R^{d_0}$ and $\cY \subseteq \R^{d_{k+1}}$ denote the input feature space and the output label space, respectively. Let $\cD$ denote the joint probability distribution on $\cX \times \cY.$ We assume that $\bE[\x\x^\top]$ has full rank. Given a training set $\{ \x_i,\y_i \}_{i=1}^n$ drawn i.i.d. from the distribution $\cD$, the goal of the learning problem is to minimize the \emph{population risk} under the squared loss $L(\{\W_{i}\}):=\bE_{\x,\y}[\| \y - \W_{k+1 \to 1}\x\|^2]$. Note that the population risk $L$ depends only on the network map and not the specific implementations of it, i.e. $L(\{ \W_i \}) = L(\{ \W'_i \})$ for all $\W_{k+1}\cdots\W_1 = \W'_{k+1}\cdots\W'_1$. For that reason, with a slight abuse of notation we write $L(\W_{k+1 \to 1}):=L(\{ \W_i \})$.

 Dropout is an iterative procedure wherein at each iteration each node in the network is dropped independently and identically according to a Bernoulli random variable with parameter $\theta$. Equivalently, we can view dropout, algorithmically, as an instance of stochastic gradient descent for minimizing the following objective over ${\{ \W_i \}}$:  
\begin{equation}\label{eq:dropout_obj}
{L_\theta(\{\W_{i}\})} := {\bE}_{(\x,\y,\{\b_i\})} [\| \y - \bar\W_{k+1\to 1} \x \|^2], 
\end{equation}
where $\bar\W_{i \to j} := \frac{1}{\theta^k}\W_{i} \B_{i-1}\W_{i-1} \cdots \B_j\W_{j}$, and  $\B_l=\diag{[\b_l(1), \ldots, \b_l(d_l)]}$ represents the dropout pattern in the $l^\textrm{th}$ layer with Bernoulli random variables on the diagonal; if $\B_l(i,i)=0$ then all paths from the input to the output that pass through the $i^\textrm{th}$ hidden node in the $l^\textrm{th}$ layer are turned ``off'', i.e., those paths have no contribution in determining the output of the network for that instance of the dropout pattern; we refer to the parameter $1-\theta$ as the \emph{dropout rate}. $\bar\W_{i\to j}$ is an unbiased estimator of $\W_{i\to j}$, i.e. $\E_{\{\b_i \}}[\bar\W_{i\to j}] = \W_{i \to j}$.

We say that the dropout algorithm \emph{succeeds} in training a network if it returns a map $\W_{k+1\to 1}$ that (approximately) minimizes $L_\theta$. In this paper, the central question under investigation is to understand \emph{which network maps/architectures is a successful dropout training biased towards.}

To answer this question, we begin with the following simple observation that
\begin{align*}
\displaystyle \!L_\theta(\{ \W_i \}) \! =\! L(\{ \W_i \})\! +\! {\bE}_{(\x, \{\b_i\})} \!\| \W_{k+1\to 1}\x - \bar\W_{k+1\to 1}\x \|^2
\end{align*} 
In other words, the dropout objective is composed of the population risk $L(\{ \W_i \})$ plus an \emph{explicit regularizer} $R(\{ \W_i \}):=\bE_{(\x,\y,\{\b_i\})} [\| \W_{k+1\to 1}\x - \bar\W_{k+1\to 1}\x \|^2]$ induced by dropout. Denoting the second moment of $\x$ by $\C:=\bE[\x\x^\top]$, we note that $R(\{\W_i \})=\bE_{\{\b_i\}} [\| \W_{k+1\to 1} - \bar\W_{k+1\to 1} \|_\C^2]$. Since any stochastic network map specified by $\bar\W_{k+1\to 1}$ is an unbiased estimator of the network map specified by $\W_{k+1\to 1}$, the explicit regularizer captures the variance of the network implemented by the weights $\{  \W_i \}$ under Bernoulli perturbations. By minimizing this variance term, dropout training aims at \emph{breaking co-adaptation between hidden nodes} -- it biases towards networks whose random sub-networks yield similar outputs~\cite{srivastava2014dropout}.

If $\{ \W^*_i \}$ is an infimum of~\eqref{eq:dropout_obj}, then it minimizes the explicit regularizer among all implementations of the network map, $\M=\W^*_{k+1} \cdots \W^*_1$, i.e., 
$R(\{ \W^*_i \}) = \infim{\W_{k+1}\cdots\W_1 = \M}{R(\{ \W_i \})}$.
We refer to the infimum of the explicit regularizer over all implementations of a given network map $\M$ as the \emph{induced regularizer}:
\begin{align}\label{eq:induced_reg}
\Theta(\M) := \infim{\W_{k+1}\cdots\W_1 = \M}{R(\{ \W_i \})}
\end{align}
The domain of the induced regularizer $\Theta$ is the linear maps implementable by the network, i.e., the set $\{ \M: \rank{\M} \leq \min_{i}{d_i} \}$. Since the infimum of the induced regularizer is always attained (see Lemma~\ref{lem:spectral} in the appendix), we can equivalently study the following problem to understand the solution to Problem~\ref{eq:dropout_obj} in terms of the network map:
\begin{align}\label{eq:optim}
\minim{\M}{L(\M) + \Theta(\M)}, \quad \rank{\M} \leq \minim{i\in [k+1]}{d_i}. 
\end{align}

To characterize which networks are preferable by dropout training, one needs to understand the explicit regularizer $R$, understand the induced regularizer $\Theta$, and explore the global minima of Problem~\ref{eq:optim}. In this regard, we make several important contributions summarized as follows.
\begin{enumerate}
\item We derive the closed form expression for the explicit regularizer $R(\{ \W_i \})$ induced by dropout training in deep linear networks. The explicit regularizer is comprised of the $\ell_2$-path regularizer as well as other rescaling invariant sub-regularizers.
\item We show that the convex envelope of the induced regularizer is proportional to the squared nuclear norm of the network map, generalizing a similar result for matrix factorization~\cite{Cavazza2017analysis} and single hidden layer linear networks~\cite{mianjy2018implicit}. Furthermore, we show that the induced regularizer equals its convex envelope if and only if the network is \emph{equalized}, a notion that quantitatively measures \emph{co-adaptation} between hidden nodes~\cite{mianjy2018implicit}. 
 
\item We completely characterize the global minima of the dropout objective $L_{\theta}$ in Problem~\ref{eq:dropout_obj} despite the objective being non-convex, under a simple eigengap condition (see Theorem~\ref{thm:main_1d}). This gap condition depends on the model, the data distribution, the network architecture and the dropout rate, and is always satisfied by deep linear network architectures with one output neuron. 
\item We empirically verify our theoretical findings. 
\end{enumerate}

The rest of the paper is organized as follows. In Section~\ref{sec:main_results}, we present the main results of the paper. In Section~\ref{sec:proofs}, we discuss the proof ideas and the key insights. Section~\ref{sec:exp} details the experimental results and Section~\ref{sec:disc} concludes with a discussion of future work. We refer the reader to Table~\ref{tab:notation} for a quick reference to the most useful notation. 

%% file: main_results.tex
\section{Main Results}\label{sec:main_results}
\input{explicit_reg}

\input{induced_reg}
\input{optimality}

%% file: explicit_reg.tex
\subsection{The explicit regularizer}\label{sec:induced_reg}

\tikzset{
  every neuron/.style={
    circle,
    draw,
    minimum size=.5cm
  },
  neuron missing/.style={
    draw=none, 
    scale=2,
    text height=0.333cm,
    execute at begin node=\color{black}$\vdots$
  },
}

\begin{figure*}[ht!]
\begin{tikzpicture}[x=1.5cm, y=1.5cm, >=stealth]

\foreach \m/\l [count=\y] in {1,2,missing,3}
  \node [every neuron/.try, neuron \m/.try] (input-\m) at (0,2-\y) {};

\foreach \m/\l [count=\y] in {1,2,3,missing,4}
  \node [every neuron/.try, neuron \m/.try] (hidden1-\m) at (1,2.5-\y) {};

\foreach \m/\l [count=\y] in {1,2,3,missing,4}
  \node [every neuron/.try, neuron \m/.try] (hidden2-\m) at (2,2.5-\y) {};

\foreach \m/\l [count=\y] in {1,2,3,missing,4}
  \node [every neuron/.try, neuron \m/.try] (hidden3-\m) at (3,2.5-\y) {};

\foreach \m/\l [count=\y] in {1,2,3,missing,4}
  \node [every neuron/.try, neuron \m/.try] (hidden4-\m) at (4,2.5-\y) {};

\foreach \m/\l [count=\y] in {1,2,3,missing,4}
  \node [every neuron/.try, neuron \m/.try] (hidden5-\m) at (5,2.5-\y) {};

\foreach \m/\l [count=\y] in {1,2,3,missing,4}
  \node [every neuron/.try, neuron \m/.try] (hidden6-\m) at (6,2.5-\y) {};

\foreach \m/\l [count=\y] in {1,2,3,missing,4}
  \node [every neuron/.try, neuron \m/.try] (hidden7-\m) at (7,2.5-\y) {};

\foreach \m/\l [count=\y] in {1,2,3,missing,4}
  \node [every neuron/.try, neuron \m/.try] (hidden8-\m) at (8,2.5-\y) {};

\foreach \m/\l [count=\y] in {1,2,3,missing,4}
  \node [every neuron/.try, neuron \m/.try] (hidden9-\m) at (9,2.5-\y) {};

\foreach \m [count=\y] in {1,2,missing,3}
  \node [every neuron/.try, neuron \m/.try ] (output-\m) at (10,2-\y) {};

\foreach \l [count=\i] in {1,2,d_0}
  \draw [<-] (input-\i) -- ++(-.75,0)
    node [above, midway] {$x[\l]$};

\node  at (hidden2-2) {\color{red}$i_1$};
\node at (hidden5-3) {\color{red}$i_2$};
\node at (hidden7-2) {\color{red}$i_3$};

\foreach \l [count=\i] in {1,2,d_{k+1}}
  \draw [->] (output-\i) -- ++(.75,0)
    node [above, midway] {$y[\l]$};

\foreach \i in {1,...,3}
  \foreach \j in {1,...,4}
    \draw [draw=green,->] (input-\i) -- (hidden1-\j);

\foreach \i in {1,...,4}
  \foreach \j in {1,3,4}
    \draw [draw=gray,dashed,->] (hidden1-\i) -- (hidden2-\j);

\foreach \i in {1,...,4}
  \foreach \j in {2}
    \draw [draw=green,->] (hidden1-\i) -- (hidden2-\j);

\foreach \i in {1,3,4}
  \foreach \j in {1,...,4}
    \draw [draw=gray,dashed,->] (hidden2-\i) -- (hidden3-\j);

\foreach \i in {2}
  \foreach \j in {1,...,4}
    \draw [draw=red,->] (hidden2-\i) -- (hidden3-\j);

\foreach \i in {1,...,4}
  \foreach \j in {1,...,4}
    \draw [draw=red,->] (hidden3-\i) -- (hidden4-\j);

\foreach \i in {1,...,4}
  \foreach \j in {1,2,4}
    \draw [draw=gray,dashed,->] (hidden4-\i) -- (hidden5-\j);

\foreach \i in {1,...,4}
  \foreach \j in {3}
    \draw [draw=red,->] (hidden4-\i) -- (hidden5-\j);

\foreach \i in {3}
  \foreach \j in {1,...,4}
    \draw [draw=red,->] (hidden5-\i) -- (hidden6-\j);

\foreach \i in {1,2,4}
  \foreach \j in {1,...,4}
    \draw [draw=gray,dashed,->] (hidden5-\i) -- (hidden6-\j);

\foreach \i in {1,...,4}
  \foreach \j in {2}
    \draw [draw=red,->] (hidden6-\i) -- (hidden7-\j);

\foreach \i in {1,...,4}
  \foreach \j in {1,3,4}
    \draw [draw=gray,dashed,->] (hidden6-\i) -- (hidden7-\j);

\foreach \i in {1,3,4}
  \foreach \j in {1,...,4}
    \draw [draw=gray,dashed,->] (hidden7-\i) -- (hidden8-\j);

\foreach \i in {2}
  \foreach \j in {1,...,4}
    \draw [draw=blue,->] (hidden7-\i) -- (hidden8-\j);

\foreach \i in {1,...,4}
  \foreach \j in {1,...,4}
    \draw [draw=blue,->] (hidden8-\i) -- (hidden9-\j);

\foreach \i in {1,...,4}
  \foreach \j in {1,...,3}
    \draw [draw=blue,->] (hidden9-\i) -- (output-\j);

\node [green, align=center, above] at (0,2) {Input\\layer};
\node [red, align=center, above] at (2,2) {$j_1$-th \\pivot};
\node [red, align=center, above] at (5,2) {$j_2$-th \\pivot};
\node [red, align=center, above] at (7,2) {$j_3$-th \\pivot};
\node [blue, align=center, above] at (10,2) {Output \\layer};

\end{tikzpicture}
\caption{\small \label{fig:reg}Illustration of the explicit regularizer as given in Proposition~\ref{prop:reg} for $k=9$, $l=3$, $(i_1,i_2,i_3)=(2,3,2)$ and $(j_1,j_2,j_3)=(2,5,7)$. The {\color{green} head} term {\color{green}$\alpha_{j_1,i_1}^2$} corresponds to the summation over the product of the weights on any pairs of path from an input node to $i_1$-th node in the $j_1$-th hidden layer. Similarly, the {\color{blue} tail} term {\color{blue}$\gamma_{j_l, i_l}^2$} accounts for the product of the weights along any pair of path between the output and the $i_l$-th node in the $j_l$-th hidden layer. Each of the {\color{red} middle} terms {\color{red}$\beta_p^2$}, accumulates the product of of the weights along pair of path between $i_p$-th node in the $(j_p+1)$-th hidden layer and the $i_{p+1}$-th node in the $j_{p+1}$-th hidden layer.}
\end{figure*}
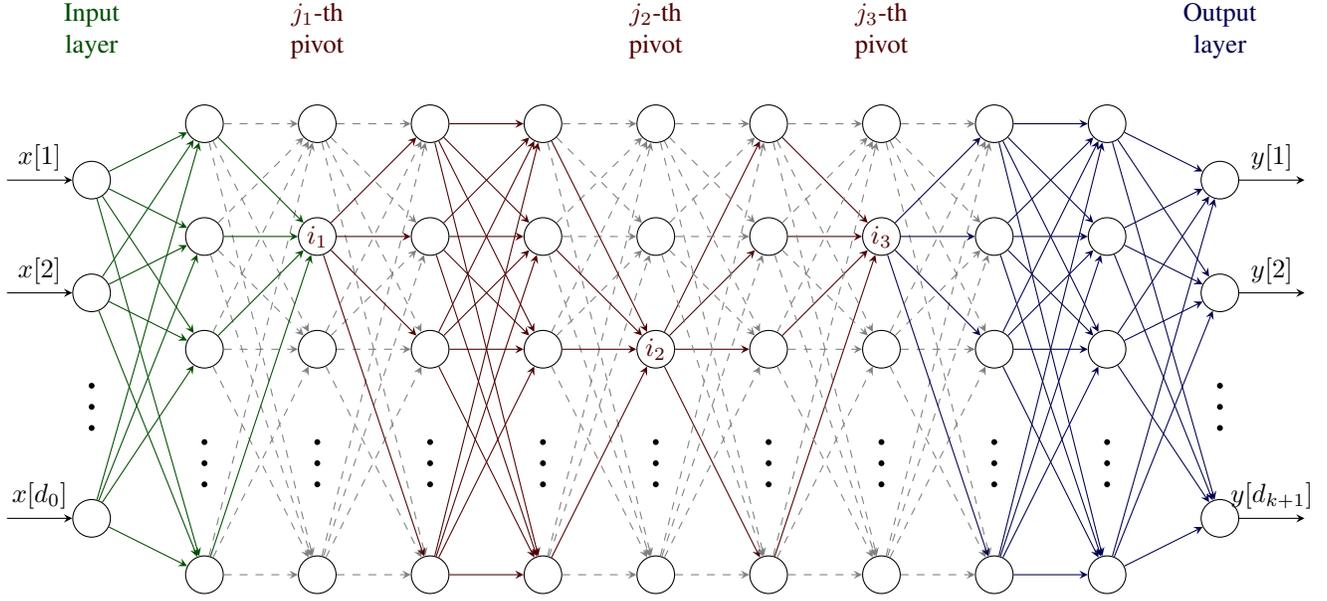
In this section, we give the closed form expression for the explicit regularizer $R(\{ \W_i \})$, and discuss some of its important properties.
\begin{prop}\label{prop:reg}
 The explicit regularizer is composed of $k$ sub-regularizers:
$R(\{ \W_{i}\}) = \sum_{l \in [k]} \lambda^l R_l(\{\W_i\})$, where $\lambda := \frac{1-\theta}{\theta}$. 
Each of the sub-regularizers has the form:
\begin{align*}
R_l(\{\W_i\})= \sum_{\substack{ (j_l,\ldots , j_1) \\ \in {[k] \choose l} }}\sum_{\substack{(i_l,\ldots,i_1) \\ \in [d_{j_l}]\times\cdots \times [d_{j_1}]}} {\color{green} \alpha_{j_1,i_1}^2 } {\color{red} \prod_{p=1\cdots l-1} \beta_p^2 } {\color{blue} \gamma_{j_l,i_l}^2}
\end{align*}
where ${\color{green} \alpha_{j_1,i_1}} := {\color{green}\|  \W_{j_1 \to 1}(i_1,:) \|_\C}$, ${\color{red} \beta_{p}} := {\color{red}\W_{j_{p+1}  \to j_{p}+1}(i_{p+1},i_p)}$, and
${\color{blue} \gamma_{j_l,i_l}} := {\color{blue}\| \W_{k+1 \to j_l+1}(:,i_l) \|}$.
\end{prop}

\paragraph{Understanding the regularizer.}
For simplicity, we assume here the case where the data distribution is whitened, i.e. $\C = \I$. This assumption is by no means restrictive, as we can always redefine $\W_1 \gets \W_1 \C^{\frac12}$ to absorb the second moment the first layer. Moreover, it is a common practice to whiten the data as a preprocessing step.

The $l$-th sub-regularizer, i.e. $R_l(\{\W_i\})$, partitions the network graph (see Figure~\ref{fig:reg}) into $l+1$ subgraphs. This partitioning is done via the choice of $l$ \emph{pivot layers}, a set of $l$ distinct hidden layers, indexed by $(j_1,\ldots, j_l) \in {[k] \choose l}$. The sub-regularizer enumerates over all such combinations of pivot layers, and \emph{pivot nodes} within them indexed by $(i_1,\ldots,i_l)\in [d_{j_1}]\times\cdots \times [d_{j_l}]$. For a given set of pivot layers and pivot nodes, the corresponding summand in the sub-regularizer is a product of three types of terms: a ``{\color{green}head}'' term {\color{green}$\alpha_{j_1,i_1}$}, ``{\color{red}middle}'' terms {\color{red}$\beta_p, \ p \in [l-1]$}, and ``{\color{blue}tail}'' terms {\color{blue}$\gamma_{j_l,i_l}$}. It is easy to see that each of the head, middle and tail terms computes a summation over product of the weights along certain walks on the (undirected) graph associated with the network (see Figure~\ref{fig:reg}). For instance, a {\color{green}head} term
\begin{align*}
{\color{green}\alpha_{j_1,i_1}}&=\sum_{i_0\in[d_0]}\sum_{\substack{i'_1,i'_2,\ldots,i'_{j_1 -1}\\ i''_{j_1 -1},\ldots, i''_2, i''_1}}\W_1(i'_1,i_0)\W_2(i'_2,i'_1)\cdots \\
& \W_{j_1}(i_1,i'_{j_l-1}) \W_{j_1}(i_1,i''_{j_l-1})\cdots \W_2(i''_2,i''_1)\W_1(i''_1,i_0),
\end{align*}
is precisely the sum of the product of all weights along all walks from $i_0$ in the input layer to $i_1$ in layer $j_1$ and back to $i_0$, i.e., walks from $i_0 \xrightarrow{i'_1,i'_2,\ldots,i'_{j_1 -1}} i_1 \xrightarrow{i''_{j_1 -1},\ldots, i''_2, i''_1} i_0$. Similarly,  
{\color{red}middle} terms
are the sum of the product of the weights along $i_p \xrightarrow{i'_1,i'_2,\ldots,i'_{j_1 -1}} i_{p+1} \xrightarrow{i''_{j_1 -1},\ldots, i''_2, i''_1} i_p$. 

A few remarks are in order.
\begin{remark}
For $k=1$, the explicit regularizer reduces to $$R(\W_2,\W_1)=\lambda\sum_{i=1}^{d_1}{ {\color{green}\| \W_1(:,i) \|^2} {\color{blue}\| \W_2(i,:) \|^2}},$$ which recovers the regularizer studied by the previous work of~\citet{Cavazza2017analysis} and \citet{mianjy2018implicit} in the setting of matrix factorization and single hidden layer linear networks, respectively.
\end{remark}
\begin{remark}
All sub-regularizers, and consequently the explicit regularizer itself are \emph{rescaling invariant}. That is, for any given implementation $\{ \W_i \}$, and any sequence of scalars $\alpha_1,\ldots,\alpha_{k+1}$ such that $\prod_i \alpha_i = 1$, it holds that $R_l(\{  \W_i \}) = R_l(\{  \alpha_i \W_i\})$. In particular, $R_k$ equals
\begin{align*}
R_k(\{ \W_i \}) &= \sum_{i_k,\ldots,i_1} {\color{green}\| \W_{1}(i_1,:) \|^2} {\color{red}\W_2(i_2,i_1)^2} \\
&{\color{red}\W_3(i_3,i_2)^2 \cdots  \W_k(i_k,i_{k-1})^2} {\color{blue}\| \W_{k+1}(:,i_k) \|^2}.
\end{align*}
Note that $R_k(\{ \W_i \}) = \psi_2^2(\W_{k+1},\ldots,\W_{1})$, the \emph{$\ell_2$-path regularizer}, which which has been recently studied in~\cite{neyshabur2015norm} and~\cite{neyshabur2015path}.
\end{remark}

%% file: induced_reg.tex
\subsection{The induced regularizer}\label{sec:multi}
In this section, we study the induced regularizer as given by the optimization problem in Equation~\eqref{eq:induced_reg}. 
We show that the convex envelope of $\Theta$ factors into a product of two terms: a term that only depends on the network architecture and the dropout rate, and a term that only depends on the network map. These two factors are captured by the following definitions.

\begin{definition}[effective regularization parameter]\label{def:flow} For given $\{ d_i \}$ and $\lambda$, we refer to the following quantity as the effective regularization parameter:
\begin{align*}
\nu_{\{\d_i\}} := \sum_{l\in [k]}\sum_{ (j_l,\ldots , j_1) \in {[k] \choose l}}{\frac{\lambda^l}{\prod_{i\in [l]} d_{j_i}}}.
\end{align*}
We drop the subscript $\{d_i\}$ whenever it is clear from the context.
\end{definition}
The effective regularization parameter naturally arises when we lowerbound the explicit regularizer (see Equation~\eqref{eq:lb_def}). It is only a function of the network architecture and the dropout rate and does not depend on the weights -- it increases with the dropout rate and the depth of the network, but decreases with the width. 
\begin{definition}[equalized network]\label{def:equalized}
A network implemented by $\{ \W_i \}_{i=1}^{k+1}$ is said to be \emph{equalized} if $\| \W_{k+1}\cdots \W_1  \C^{\frac12}  \|_*$
is equally distributed among all the summands in Proposition~\ref{prop:reg}, i.e. for any $l\in [k]$, $(j_l,\ldots , j_1) \in {[k] \choose l}$, and $(i_l,\ldots,i_1)\in [d_{j_l}]\times\cdots\times[d_{j_1}]$ it holds that $$\left| {\color{green}\alpha_{j_1,i_1}}   {\color{red}\beta_1 \cdots  \beta_{l-1}}  {\color{blue}\gamma_{j_l,i_l}} \right| = \frac{\| \W_{k+1}\cdots \W_1  \C^{\frac12}  \|_*}{\Pi_l d_{j_l}}.$$
\end{definition}

We are now ready to state the main result of this section.
Recall that the convex envelope of a function is the largest convex under-estimator of that function. 
We show that irrespective of the architecture, the convex envelope of the induced regularizer is proportional to the squared nuclear norm of the network map times the principal root of the second moment.
\begin{theorem}[Convex Envelope]\label{thm:envelope}
For any architecture $\{ d_i \}$ and any network map $\M\in \R^{d_{k+1}\times d_0}$ implementable by that architecture, it holds that: 
\begin{equation*}
\Theta^{**}(\M) = \nu_{\{ d_i \}} \| \M  \C^{\frac12}  \|_*^2
\end{equation*}
Furthermore, $\Theta(\M)=\Theta^{**}(\M)$ if and only if the network is equalized.
\end{theorem}
This result is particularly interesting because it connects dropout, an algorithmic heuristic to avoid overfitting, to nuclear norm regularization, which is a classical regularization method with strong theoretical foundations. We remark that a result similar to Theorem~\ref{thm:envelope} was recently established for matrix factorization~\cite{Cavazza2017analysis}. 

%% file: optimality.tex
\subsection{Global optimality}\label{sec:1d}
Theorem~\ref{thm:envelope} provides a sufficient and necessary condition under which the induced regularizer equals its convex envelope. If any network map can be implemented by an equalized network, then $\Theta(\M)=\Theta^{**}(\M)=\nu_{\{ d_i \}}\| \M\C^\frac12 \|_*^2$, and the learning problem in Equation~\eqref{eq:optim} is a convex program. In particular, for the case of linear networks with single hidden layer, \citet{mianjy2018implicit} show that any network map can be implemented by an equalized network, which enables them to characterize the set of global optima under the additional generative assumption $\y=\M\x$. However, it is not clear if the same holds for general deep linear networks since the regularizer here is more complex. Nonetheless, the following result provides a sufficient condition under which global optima of $L_\theta(\{ \W_i \})$ are completely characterized.
\begin{theorem}\label{thm:main_1d}
Let $\C_{\y\x}:=\bE[\y\x^\top]$ and $\C:=\bE[\x\x^\top]$, and denote $\bar\M:=\C_{\y\x}\C^{-\frac12}$. If $\sigma_1(\bar\M) - \sigma_2(\bar\M) \geq \frac{1}{\nu}\sigma_2(\bar\M)$, then $\M^*$, the global optimum of Problem~\ref{eq:optim}, is given by $$\W^*_{k+1 \to 1} = \cS_{\frac{\nu\sigma_1(\bar\M)}{1+\nu}}(\bar\M)\C^{-\frac12},$$ where $\cS_\alpha(\bar\M)$ shrinks the spectrum of matrix $\bar\M$ by $\alpha$ and thresholds it at zero. Furthermore, it is possible to implement $\M^*$ by an equalized network $\{ \W^*_i\}$ which is a global optimum of $L_\theta(\{ \W_i \})$.
\end{theorem}
The gap condition in the theorem above can always be satisfied (e.g. by increasing the dropout rate or the depth, or decreasing the width) as long as there exists a gap between the first and the second singular values of $\bar\M$. Moreover, for the special case of deep linear networks with one output neuron~\cite{ji2018gradient,nacson2018convergence}, this condition is always satisfied since $\bar\M\in \R^{1 \times d_0}$ and $\sigma_2(\bar\M)=0$. 
\begin{corollary}
\label{cor:1d}
Consider the class of deep linear networks with a single output neuron. Let $\{ \W^*_i \}$ be a minimizer of $L_\theta$. For any architecture $\{ d_i \}$ and any network map $\W_{k+1\to 1}$, it holds that: (1) $\Theta(\W_{k+1\to 1}) = \nu\|  \W_{k+1\to 1} \|_\C^2,$
(2) $\W^*_{k+1 \to 1} \!\!=\!\! \frac{1}{1+\nu} \C_{\y\x}$, (3) the network specified by $\{ \W^*_i \}$ is equalized. 
\end{corollary}

We conclude this section with a remark. We know from the early work of~\cite{srivastava2014dropout} that feature dropout in linear regression is closely related to ridge regression. Corollary~\ref{cor:1d} generalizes the results of \cite{srivastava2014dropout} to deep linear networks, and establishes a similar connection between dropout training and ridge regression. 

%% file: proofs.tex
\section{Proof Ideas}\label{sec:proofs}
Here, we sketch proofs of the main results from Section~\ref{sec:main_results}; complete proofs are deferred to the supplementary.\vspace{-5pt}
\paragraph{Sketch of the Proof of Theorem~\ref{thm:envelope}} The key steps are: 
\begin{enumerate}
\item \label{item1} First, in Lemma~\ref{lem:reg_lb}, we show that for any set of weights $\{ \W_i \}$, it holds that $R(\{ \W_i \}) \geq \nu_{\{ d_i \}}\| \W_{k+1\to 1}  \C^{\frac12}  \|_*^2$. In particular, this implies that $\Theta(\M)\geq \nu_{\{ d_i \}}\| \M  \C^{\frac12}  \|_*^2$ holds for any $\M$.
\item \label{item2}Next, in Lemma~\ref{lem:fenchel}, we show that $\Theta^{**}(\M)\leq \nu_{\{ d_i \}}\| \M \C^{\frac12}  \|_*^2$ holds for all $\M$.
\item The claim is implied by Lemmas~\ref{lem:reg_lb} and~\ref{lem:fenchel}, and the fact that $\| \cdot \|_*^2$ is a convex function. 
\end{enumerate}

\begin{lemma}\label{lem:reg_lb}
Let $\{ \W_i \}$ be an arbitrary set of weights. The explicit regularizer $R(\{\W_{i}\})$ satisfies
$$R(\{\W_{i}\}) \geq \nu_{\{ d_i \}} \| \W_{k+1}\W_k \cdots \W_1 \C^{\frac12}  \|_*^2,$$
and the equality holds if and only if the network is equalized. 
\end{lemma}
We sketch the proof for isotropic distributions ($\C=\I$), and emphasize the role of equalization and effective regularization parameter. 

\paragraph{Sketch of the Proof of Lemma~\ref{lem:reg_lb}}
We show that each term in the explicit regularizer is lower bounded by some multiple of the square of the nuclear norm of the linear map implemented by the network. We begin by lower bounding a particular summand in the expansion of $R_l(\{\W_i\})$ given in Proposition~\ref{prop:reg} with indices $j_1,\ldots,j_l$: 
\begin{align*}
&R_{\{j_l,\ldots,j_1\}}(\{\W_i\}) = \sum_{ i_l,\ldots,i_1}  {\color{green} \alpha_{j_1,i_1}^2}   {\color{red}\prod_{p=1\cdots l-1} \beta_p^2}  {\color{blue}\gamma_{j_l,i_l}^2 }\\
& \geq  \frac{1}{\Pi_l d_{j_l}} \left( \sum_{ i_l,\ldots,i_1}  \left|{\color{green} \alpha_{j_1,i_1}}  {\color{red} \beta_1 \cdots  \beta_{l-1} } {\color{blue} \gamma_{j_l,i_l} } \right| \right)^2,
\end{align*}
where the inequality follows due to Cauchy-Schwartz inequality and holds with equality if and only if all the summands in $\phi:=\sum_{ i_l,\ldots,i_1} \left|{\color{green} \alpha_{j_1,i_1}}  {\color{red} \beta_1 \cdots  \beta_{l-1} } {\color{blue} \gamma_{j_l,i_l} } \right|$ are equal for all $i_l,\ldots,i_1 \in [d_{j_l}] \times \cdots \times  [d_{j_1}]$. At the same time, we have that 
\begin{align*}
& \phi = \sum_{i_l,\ldots,i_1} {\color{blue}\| \W_{k+1\to j_l+1}(:,i_l)\|} {\color{red} \left| \beta_1 \cdots  \beta_{l-1} \right|} {\color{green}  \|\W_{j_1 \to 1}(i_1,:) \|} \\
& \stackrel{(a)}{=} \sum_{i_l,\ldots,i_1} \|{\color{blue} \W_{k+1\to j_l+1}(:,i_l)} {\color{red}  \beta_1 \cdots  \beta_{l-1}} {\color{green} \W_{j_1 \to 1}(i_1,:)^\top} \|_* \\
& \geq  \| \sum_{i_l,\ldots,i_1} {\color{blue} \W_{k+1\to j_l+1}(:,i_l)} {\color{red} \beta_1 \cdots  \beta_{l-1} } {\color{green} \W_{j_1 \to 1}(i_1,:)^\top} \|_* \\
&\stackrel{(b)}{=} \| \W_{k+1}\cdots \W_1 \|_*
\end{align*}
where $(a)$ follows since the outer product inside the nuclear norm has rank equal to $1$, the inequality is due to the triangle inequality for matrix norms, and $(b)$ follows trivially. In fact, the inequality above holds if each of the summands in $(a)$ are equal to 
\begin{equation}\label{eq:equalized}
\left|{\color{green} \alpha_{j_1,i_1}  } {\color{red} \beta_1  \cdots   \beta_{l-1} } {\color{blue} \gamma_{j_l,i_l} }\right| = \frac{\| \W_{k+1}\cdots \W_1 \|_*}{\Pi_l d_{j_l}}
\end{equation}
for all $i_l,\ldots,i_1 \in [d_{j_l}] \times \cdots \times  [d_{j_1}]$. Each of the sub-regularizers can be lowerbounded by noting that $R_l(\{ \W_i \})=\sum_{ j_l, \ldots, j_1}R_{\{ j_l, \ldots, j_1 \}}(\{ \W_i \})$.

Lemma~\ref{lem:reg_lb} is central to our analysis for two reasons. First, it gives a sufficient and necessary condition for the induced regularizer to equal the square of the nuclear norm of the network map. This also motivates the concept of equalized networks in Definition~\ref{def:equalized}. We note that for the special case of single hidden layer linear networks, i.e., for k=1, this lower bound can always be achieved~\cite{mianjy2018implicit}; it remains to be seen whether the lower bound can be achieved for deeper networks. Second, summing over $\{ j_l, \ldots, j_1 \} \in {[k] \choose l}$, we conclude that 
\begin{equation}\label{eq:lb_def}
R_l(\{ \W_i \}) \geq\! \sum_{ j_l, \ldots, j_1}\! \frac{\| \W_{k+1\to 1} \|_*^2}{\prod_l d_{j_l}} =: LB_l(\{ W_i \}). 
\end{equation}
The right hand side above is the lowerbound for $l$-th subregularizer, denoted by $LB_l$. Summing over $l\in[k]$, we get the following lowerbound on the explicit regularizer\vspace*{-5pt}
\begin{equation}
R(\{ \W_i \})\ \geq\ \ \| \W_{k+1 \to 1} \|_*^2\ \ \underbrace{\sum_{l}\sum_{j_l, \ldots, j_1}\!\!\frac{
\lambda^l}{\prod_l d_{j_l}}}_{\nu_{\{ d_i \}}} \vspace*{-5pt}
\end{equation}
which motivates the notion of \emph{effective regularization parameter} in Definition~\ref{def:flow}.  
As an immediate corollary of Lemma~\ref{lem:reg_lb}, it holds that for any matrix $\M$ we have that $\Theta(\M) \geq \nu_{\{ d_i \}}\| \M \|_*^2$. We now focus on the biconjugate of the induced regularizer, and show that it is upper bounded by the same function, i.e. the effective regularization parameter times the square of the nuclear norm of the network map.
\begin{lemma}\label{lem:fenchel}
For any architecture $\{ d_i \}$ and any network map $\M$, it holds that $\Theta^{**}(\M)\leq \nu_{\{ d_i \}}\| \M \C^\frac12 \|_*^2$.
\end{lemma}\vspace*{-5pt}
To convey the main ideas of the proof, here we include a sketch for the simple case of $k=2$,  $d_1=d_2=d$ under the isotropic assumption ($\C=\I$); for the general case, please refer to the appendix. \vspace*{-5pt} 
\paragraph{Sketch of the proof of Lemma~\ref{lem:fenchel}}
First, the induced regularizer is non-negative, so the domain of $\Theta^*$ is $\R^{d_{k+1} \times d_0}$. The Fenchel dual of the induced regularizer $\Theta(\cdot)$ is given~by:\vspace*{-5pt}
\begin{align}\label{eq:fenchel}
\Theta^*(\M) 
&= \maxim{\P}{\langle \M,\P \rangle - \Theta(\P)}\nonumber\\
&= \maxim{\P}{\langle \M,\P \rangle - \minim{\substack{\W_3,\W_2,\W_1 \\ \W_{3\to 1} = \P}}{R(\W_3,\W_2,\W_1)}}\nonumber\\
&= \maxim{\W_3,\W_2,\W_1}{\langle \M,\W_{3\to 1} \rangle - R(\W_3,\W_2,\W_1)},
\end{align}
Denote the objective in~\eqref{eq:fenchel} by $\Phi(\W_3,\W_2,\W_1):=\langle \M,\W_{3\to 1} \rangle - R(\W_3,\W_2,\W_1)$. Let $(\u_1,\v_1)$ be the top singular vectors of $\M$. For any $\alpha\in \R$, consider the following assignments to the weights: $\W^\alpha_1=\alpha\u_1\1_d^\top$,$\W^\alpha_2=\1_d\1_d^\top$ and $\W^\alpha_3=\1_d\v_1^\top$. Note that $$\maxim{\W_3,\W_2,\W_1}{\Phi(\W_3,\W_2,\W_1)} \geq \maxim{\alpha}{\Phi(\W^\alpha_3,\W^\alpha_2,\W^\alpha_1)}.$$ We can express the objective on the right hand side merely in terms of $\alpha,d$ and $\|\M\|_2$ as follows: 
\begin{align*}
&R(\W^\alpha_{3},\W^\alpha_{2}, \W^\alpha_{1}) = \lambda \sum_{i=1}^{d}{ {\color{green}\| \W^\alpha_{1}(i,:) \|^2} {\color{blue}\| \W^\alpha_{3\to 2}(:,i) \|^2} }\\\vspace*{-5pt}
&+ \lambda \sum_{i=1}^{d}{ {\color{green}\| \W^\alpha_{2\to 1}(i,:) \|^2} {\color{blue}\| \W^\alpha_3(:,i) \|^2} }\\\vspace*{-5pt}
&+ \lambda^2 \sum_{i,j=1}^{d}{ {\color{green}\| \W^\alpha_{1}(i,:) \|^2} {\color{red}\W^\alpha_{2}(j,i)^2} {\color{blue}\| \W^\alpha_3(:,j) \|^2} }\\
&= 2 \lambda \alpha^2 d^3 + \lambda^2 \alpha^2 d^2.
\end{align*}
Similarly, the inner product $\langle \M, \W^\alpha_{3\to 1}\rangle$ reduces to \vspace*{-5pt}
\begin{align*}
\langle \M, \W^\alpha_{3 \to 1}\rangle 
&= \sum_{i,j=1}^{d}{ \langle \M, \W^\alpha_3(:,j) \W^\alpha_{2}(j,i)  \W^\alpha_{1}(i,:)^\top \rangle} \\
&=\sum_{i,j=1}^{d}{ \alpha \u_1^\top\M\v_1} = \alpha d^2 \| \M \|_2.
\end{align*}
Plugging back the above equalities in $\Phi$ we get:
\begin{align*}
\Phi(\W^\alpha_3,\W^\alpha_2,\W^\alpha_1) &= \alpha d^2 \|  \M \|_2 - 2 \lambda \alpha^2 d^3 - \lambda^2 \alpha^2 d^2.
\end{align*}
Maximizing the right hand side above with respect to $\alpha$
\begin{equation*}
\Theta^*(\M) \geq \maxim{\alpha}{\Phi(\W^\alpha_3,\W^\alpha_2,\W^\alpha_1)} = \frac{d^2 }{4(2\lambda d + \lambda^2)} \| \M \|_2^2.
\end{equation*}
 Since Fenchel dual is order reversing, we get 
\begin{align}{\label{eq:biconjugate}}
\Theta^{**}(\M) 
&\leq \frac{2\lambda d + \lambda^2}{d^2} \| \M \|_*^2= \nu_{\{d,d\}} \| \M \|_*^2,
\end{align}
where we used the basic duality between the spectral norm and the nuclear norm. 
Lemma~\ref{lem:reg_lb} implies that the biconjugate $\Theta^{**}(\M)$ is lower bounded by  $\nu_{\{ d,d \}} \| \M \|_*^2$, which is a convex function. On the other hand, inequality~\eqref{eq:biconjugate} shows that the biconjugate is upper bounded $\nu_{\{ d,d \}} \| \M \|_*^2$. Since square of the nuclear norm is a convex function, and that $\Theta^{**}(\cdot)$ is the largest convex function that lower bounds $\Theta(\cdot)$, we conclude that $\Theta^{**}(\M) = \nu_{\{ d,d \}} \| \M \|_*^2$.

\paragraph{Sketch of the proof of Theorem~\ref{thm:main_1d}} 
In light of Theorem~\ref{thm:envelope}, if the optimal network map $\W^*_{k+1\to 1}$, i.e. the optimum of the  problem in Equation~\ref{eq:optim} can be implemented by an equalized network, then $\Theta(\W^*_{k+1\to 1})=\Theta^{**}(\W^*_{k+1\to 1})=\nu_{\{ d_i \}}\|\W^*_{k+1\to 1} \C^\frac12\|_*^2$. Thus, the learning problem essentially boils down to the following convex program:
\begin{equation}\label{eq:related1}
\minim{\W}{\bE_{\x,\y}[\| \y - \W\x \|^2] + \nu_{\{ d_i \}}\| \W\C^\frac12 \|_*^2}.
\end{equation}
Following the previous work of~\cite{Cavazza2017analysis,mianjy2018implicit}, we show that the solution to problem~\eqref{eq:related1} can be given as $\W^* = \cS_{\alpha_\rho}(\C_{\y\x}\C^{-\frac12})\C^{-\frac12}$, where $\alpha_\rho:=\frac{\nu\sum_{i=1}^{\rho}\sigma_i(\C_{\y\x}\C^{-\frac12})}{1+\rho\nu}$, $\rho$ is the rank of $\W^*$, and $\cS_{\alpha_\rho}(\M)$ shrinks the spectrum of the input matrix $\M$ by $\alpha_\rho$ and thresholds them at zero. However, as mentioned above, it is not clear if any network map can be implemented by an equalized network.  
Nonetheless, it is easy to see that the equalization property is satisfied for rank-1 network maps. 
\begin{proposition}\label{prop:equalizable}
Let $\{ d_i \}_{i=0}^{k+1}$ be an architecture and $\M\in \R^{d_{k+1}\times d_0}$ be a rank-1 network map. Then, there exists a set of weights $\{ \W_i \}_{i=1}^{k+1}$ that implements $\M$, and is equalized. 
\end{proposition}
For example, for deep networks with single output neuron, the weights $\W_1=\frac{\1_{d_1} \w^\top }{\sqrt d_1}$ and $\W_i = \frac{\1_{d_i}\1_{d_{i-1}}^\top}{\sqrt{d_i d_{i-1}}}$ for $i \neq 1$ implements the network map $\w^\top$, and are equalized.

Denote $\bar\M:=\C_{\y\x}\C^{-\frac12}$. Equipped with Proposition~\ref{prop:equalizable}, the key here is to ensure that $\cS_\alpha(\bar\M)$ has rank equal to one. In this case, $\W^*$ will also have rank at most one and can be implemented by a network that is equalized. To that end, it suffices to have $\alpha \geq \sigma_2(\bar\M)$, which implies
$$\frac{\nu\sigma_1(\bar\M)}{1+\nu} \geq \sigma_2(\bar\M) \implies \sigma_1(\bar\M) - \sigma_2(\bar\M) \geq \frac{\sigma_2(\bar\M)}{\nu}$$
which gives the sufficient condition in Theorem~\ref{thm:main_1d}. 

%% file: exp.tex
\section{Experimental Results}\label{sec:exp}
Dropout is widely used for training modern deep learning architectures resulting in the state-of-the-art performance in numerous machine learning tasks~\cite{srivastava2014dropout,krizhevsky2012imagenet,szegedy2015going,toshev2014deeppose}. The purpose of this section is not to make a case for (or against) dropout when training deep networks, but rather verify empirically the theoretical results from the previous section.\footnote{The code for the experiments can be found at: \href{https://github.com/r3831/dln_dropout}{https://github.com/r3831/dln\_dropout}} 

For simplicity, the training data $\{ \x_i \}$ is sampled from a standard Gaussian distribution which in particular ensures that $\C=\I$. The labels $\{ \y_i \}$ are generated as $\y_i \gets \textrm{N}\x_i$, where $\textrm{N}\in \R^{d_{k+1}\times d_0}$.  
$\textrm{N}$ is composed of $\U \V^\top + \texttt{noise}$, where $\U \in \R^{d_{k+1} \times r}$, $\V \in \R^{d_0 \times r}$ are sampled from a standard Gaussian and the entries of $\texttt{noise}$ are sampled independently from a Gaussian distribution with small standard deviation. At each step of the dropout training, we use a minibatch of size $1000$ to train the network. The learning rate is tuned over the set $\{ 1,0.1,0.01 \}$. All experiments are repeated 50 times, the curves correspond to the average of the runs, and the grey region shows the standard deviation.

The experiments are organized as follows. First, since the convex envelope of the induced regularizer equals the squared nuclear norm of the network map (Theorem~\ref{thm:envelope}), it is natural to expect that dropout training performs a shrinkage-thresholding on the spectrum of $\C_{\y\x}\C^{-\frac12}=\M$ (see Lemma~\ref{lem:global} in the appendix). We experimentally verify this in Section~\ref{sec:shrinkage}. Second, in Section~\ref{sec:subregs}, we focus on the equalization property. We attest Theorem~\ref{thm:main_1d} by showing that dropout training equalizes deep networks with one output neuron.

\subsection{Spectral shrinkage and rank control}\label{sec:shrinkage}
\begin{figure}[t]
\centering
\includegraphics[width=0.49\textwidth]{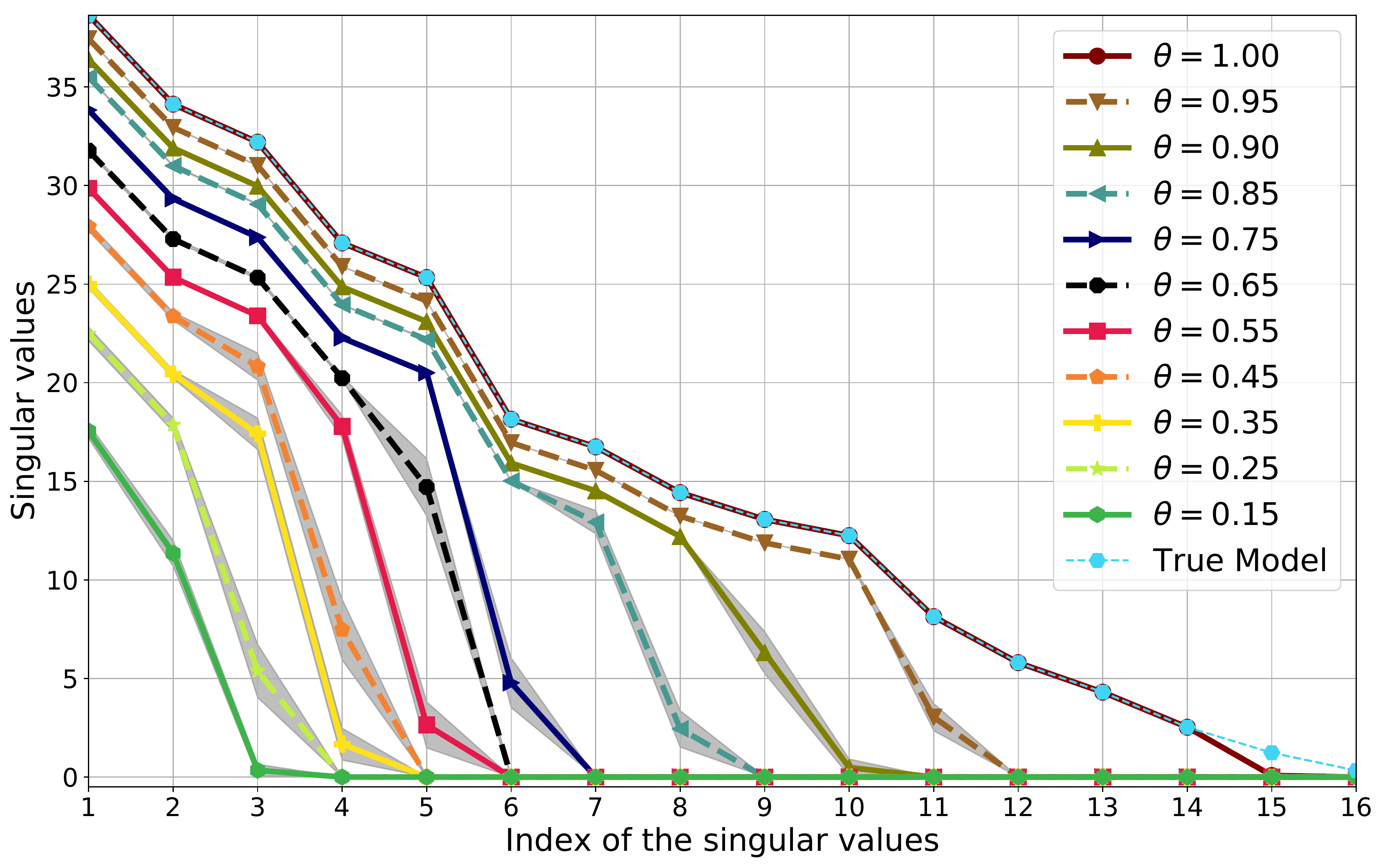}
\caption{ \label{fig:sv}Distribution of the singular values of the trained network for different values of the dropout rate $1-\theta$. It can be seen that the dropout training performs a more sophisticated form of shrinkage and thresholding on the spectrum of the model matrix $\M$.}  
\end{figure}
Note that the induced regularizer $\Theta(\M)$ is a \emph{spectral function} (see Lemma~\ref{lem:spectral} in the appendix). 
On the other hand, by Theorem~\ref{thm:envelope}, $\Theta^{**}(\M)=\nu_{\{ d_i \}}\|\M\|_*^2$. Therefore, if dropout training succeeds in finding an (approximate) minimizer of $L_\theta$, it minimizes an upperbound on the squared of the nuclear norm of the network map. Hence, it is natural to expect that the dropout training performs a shrinkage-thresholding on the spectrum of the model, much like nuclear norm regularization. Figure~\ref{fig:sv} confirms this intuition. Here, we plot the singular value distribution of the final network map trained by dropout, for different values of the dropout rate. 

As can be seen in the figure, dropout training indeed shrinks the spectrum of the model and thresholds it at zero. However, unlike the nuclear norm regularization, the shrinkage is not uniform across the singular values that are not thresholded. Moreover, note that the shrinkage parameter in Theorem~\ref{thm:main_1d} is governed by the effective regularization parameter $\nu_{\{ d_i \}}$, which strictly increases with the dropout rate. This suggests that as we increase the dropout rate (decrease $\theta$), the spectrum should be shrunk more severely, and the resulting network map should have a smaller rank. This is indeed the case as can be seen in Figure~\ref{fig:sv}.\vspace{-5pt}

\subsection{Convergence to equalized networks}\label{sec:subregs}
\begin{figure}
\centering
\includegraphics[width=0.49\textwidth]{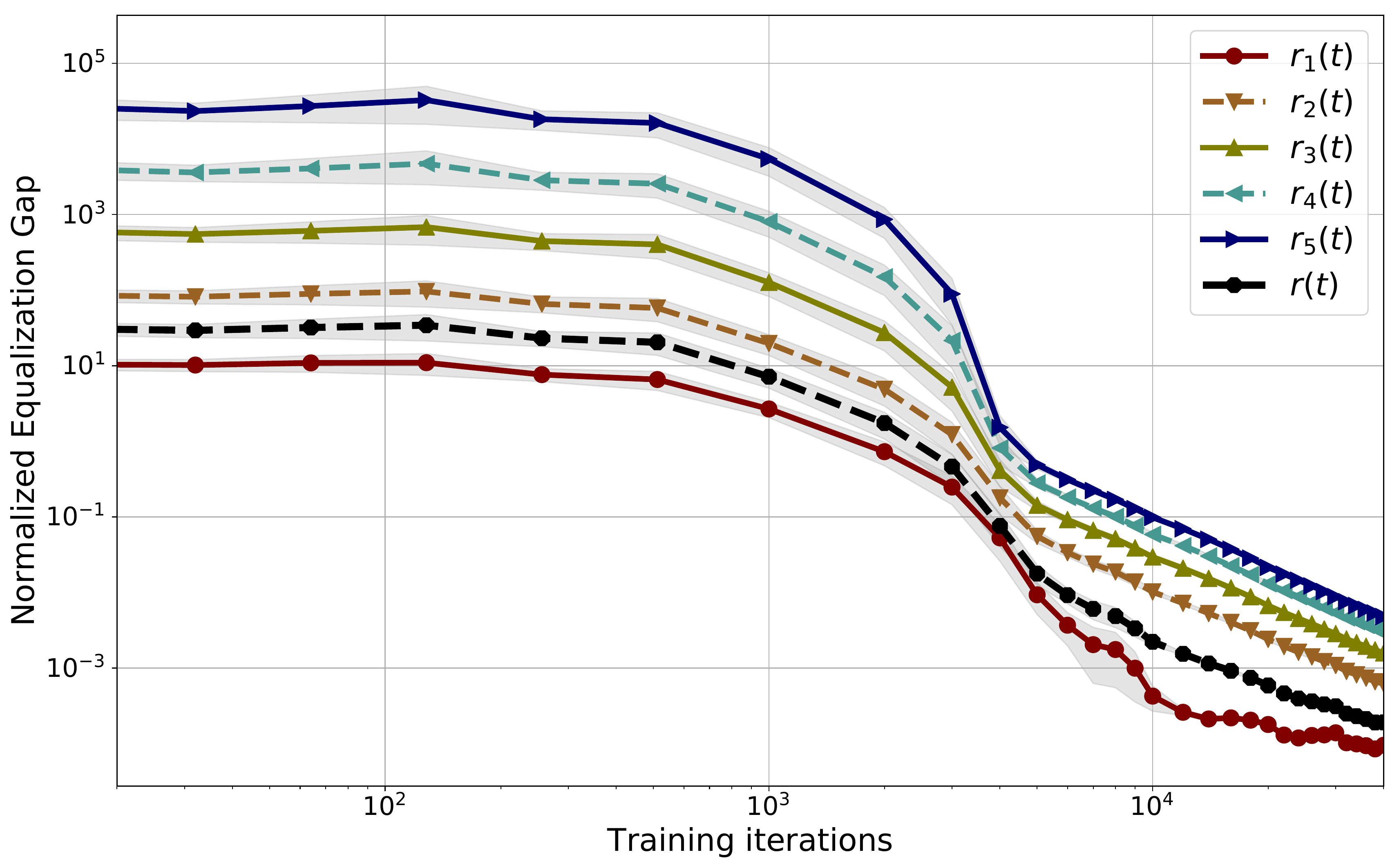}
\caption{ \label{fig:ratio}
The normalized equalization gap $r_\ell^{(t)}$, which captures the gap between the sub-regularizers and their respective lower bounds, is plotted as a function of the number of iterations. Dropout converges to the set of equalized networks. }
\end{figure}

One of the key concepts behind our analysis is the notion of equalized networks. In particular, in Lemma~\ref{lem:reg_lb} we see that if a network map can be implemented by an equalized network, then there is no gap between the induced regularizer and its convex envelope. It is natural to ask if dropout training indeed finds such equalized networks. As we will discuss, Figure~\ref{fig:ratio} answers this question affirmatively.

Recall that a network is equalized if and only if each and every sub-regularizer achieves its respective lowerbound in Equation~\ref{eq:lb_def}, i.e. $R_l(\{ \W_i \})=LB_l(\{ \W_i \})$ for all $l \in [k]$. Figure~\ref{fig:ratio} illustrates that dropout training consistently decreases the gap between the sub-regularizers and their respective lowerbounds. Here, the network has one output neuron, five hidden layers each of width 5, and input dimensionality of $d_0=5$. In Figure~\ref{fig:ratio} we plot the \emph{normalized equalization gap} 
$r_\ell^{(t)} := \frac{R_\ell(\{ \W_i^{(t)} \})}{LB_\ell(\{ \W_i^{(t)} \})}-1$
of the network under dropout training as a function of the iteration number. Similarly, we define the normalized equalization gap for the explicit regularizer $r^{(t)}=\frac{R(\{ \W_i \})}{ \Theta^{**}(\W_{k+1\to 1})}-1$. The network quickly becomes (approximately) equalized, and thereafter the trajectory of dropout training stays close to the equalized networks. We believe that this observation can be helpful in analyzing the dynamics of dropout training, which we leave for future work. 

%% file: disc.tex
\section{Discussion}\label{sec:disc}
Motivated by empirical success of dropout~\cite{srivastava2014dropout,krizhevsky2012imagenet}, there has been several studies in recent years to understand its theoretical foundations~\cite{baldi2013understanding,wager2013dropout,wager2014altitude,van2014follow, helmbold2015inductive,gal2016dropout,gao2016dropout,helmbold2017surprising,mou2018dropout,bank2018relationship}.

Previous work of~\citet{zhai2015dropout,he2016dropout,Cavazza2017analysis} and~\citet{mianjy2018implicit} study dropout training with $\ell_2$-loss in matrix factorization and shallow linear networks, respectively. The work that is most relevant to us is that of~\citet{Cavazza2017analysis,mianjy2018implicit}, whose results are extended to the case of deep linear networks in this paper.

In particular, we derive the \emph{explicit regularizer} induced by dropout, which happens to be composed of the  $\ell_2$-path regularizer and other rescaling invariant regularizers. Furthermore, we show that the convex envelope of the induced regularizer factors into an \emph{effective regularization parameter} and the square of the  nuclear norm of network map multiplied with the principal root of the second moment of the input distribution. We further highlight \emph{equalization} as a key network property under which the induced regularizer equals its convex envelope. We specify a subclass of problems satisfying the equalization property, for which we completely characterize the optimal networks that dropout training is biased towards. 

Our work suggests several interesting directions for future research. First, given the connections that we establish with the nuclear norm and the $\ell_2$-path regularization, it is natural to ask what role does the dropout regularizer play in generalization. Second, how does the dropout regularizer change for neural networks with non-linear activation functions. Finally, it is important to understand dropout in networks trained with other loss functions, especially those that are popular for various classification tasks.

%% file: appendix.tex
\newpage
\appendix
\onecolumn

\pagenumbering{gobble}

\title{Supplementary Materials for the Paper: \\``On Dropout and Nuclear Norm Regularization''}
\date{}
\maketitle

\section{Proofs of the Main Results}
In this section, we provide the complete proofs of our main results. For notational simplicity, we define 
\begin{align*}
\W_{i \to j} &:=\W_i \W_{i-1}\cdots \W_{j+1} \W_j, \\
\bar\W_{i \to j} &:=\frac1{\theta^{i-j}}\W_i \diag{\b_{i-1}}\W_{i-1}\cdots \diag{\b_{j+1}}\W_{j+1} \diag{\b_{j}}\W_j.
\end{align*}
Since the Bernoulli random vectors are i.i.d., it holds that $\bE_{\b_i}[\bar\W_{i \to j}] = \W_{i \to j}$. A quantity that shows up when analyzing dropout training under squared error is $\bE[\|\U\diag{\b}\V\x \|^2]$, where the expectation is taken with respect to $\b$, which is a Bernoulli random vector with parameter $\theta$. The following lemma gives the closed form of this expectation.
\begin{lemma}\label{lem:linear_norm}
Let $\U\in \R^{d_2\times r}$, $\V\in \R^{d_1\times r}$, and $\C := \bE[\x\x^\top]$. It holds that $$\bE[\|\U\diag{\b}\V\x \|^2]=\theta^2\bE[\|\U\V\x \|^2] + (\theta-\theta^2) \sum_{j=1}^{r}\|\u_{:j}\|^2 \|\C^{\frac12}\v_{j:}\|^2.$$
\end{lemma}
The proof can be found in~\cite{Cavazza2017analysis,mianjy2018implicit}. Nonetheless, we provide a proof here for completeness.
\begin{proof}[Proof of Lemma~\ref{lem:linear_norm}]
\begin{align*}
\bE[\|\U\diag{\b}\V\x \|^2]&=\bE \sum_{i=1}^{d_2}{\bE_\b\left(\sum_{j=1}^{r}{u_{ij}b_j\v_{j:}^\top \x}\right)^2} =\bE \sum_{i=1}^{d_2}{\bE_\b[\sum_{j,k=1}^{r}{u_{ij}u_{ik}b_j b_k (\v_{j:}^\top \x) (\v_{k:}^\top\x)}]}\nonumber\\
&=\bE \sum_{i=1}^{d_2}{\sum_{j,k=1}^{r}{u_{ij}u_{ik}(\theta^2 \1_{j\neq k} + \theta \1_{j=k}) (\v_{j:}^\top \x) (\v_{k:}^\top\x)}} =\theta^2\bE[\|\U\V\x \|^2] + (\theta-\theta^2)\bE \sum_{i=1}^{d_2}{\sum_{j=1}^{r}{u_{ij}^2 (\v_{j:}^\top \x)^2}}\nonumber\\
&=\theta^2\bE[\|\U\V\x \|^2] + (\theta-\theta^2)  \sum_{j=1}^{r}\bE[\v_{j:}^\top \x \x^\top \v_{j:}] {\sum_{i=1}^{d_2}{u_{ij}^2}} =\theta^2\bE[\|\U\V\x \|^2] + (\theta-\theta^2) \sum_{j=1}^{r} \| \C^{\frac12}\v_{j:} \|^2 \|\u_{:j}\|^2.
\end{align*}
\end{proof}

\subsection{Properties of the explicit regularizer}
Recall that training a network with dropout aims at minimizing the following \emph{dropout objective} $$L_\theta(\{ \W_{i}\}_{i=1}^{k+1})=\bE_{\substack{\b_i \sim \text{Bern}(\theta) \\ (\x,\y)\sim \cD}}[\|\y - \frac1{\theta^k}\W_{k+1}\diag{\b_k}\W_k \ldots \diag{\b_1}\W_1\x\|^2].$$ In Proposition~\ref{prop:reg} we show that this objective can be decomposed into a summation of the population loss plus an \emph{explicit regularizer}, i.e. $L_\theta(\cdot) = L(\cdot) + R(\cdot)$, and give the closed form expression for the explicit regularizer.
\begin{proof}[Proof of Proposition~\ref{prop:reg}]
We start by expanding the squared error:
\begin{align}\label{eq:der1}
L_\theta(\{ \W_{i}\}_{i=1}^{k+1})&=\bE_{\substack{\b_i \sim \text{Bern}(\theta) \\ (\x,\y)\sim \cD}}[\|\y - \bar\W_{k+1 \to 1} \x\|^2] \nonumber \\
&=\bE[\|\y\|^2] - 2\bE[\langle \bar\W_{k+1 \to 1}\x,\y\rangle] + \bE[\|\bar\W_{k+1 \to 1}\x\|^2]  \nonumber \\
&= \bE[\| \y \|^2] - 2 \bE \langle \W_{k+1\to 1}\x,\y \rangle + \frac1{\theta^{2k}}\bE[\|\W_{k+1}\diag{\b_k}\W_{k} \ldots \diag{\b_1}\W_{1}\x\|^2]  
\end{align}

We now focus on the last term in the right hand side of Equation~\eqref{eq:der1}.
\begin{align}\label{eq:der2}
\bE[\|\W_{k+1}\diag{\b_k}\W_{k} \ldots & \diag{\b_1}\W_{1}\x\|^2] = \bE[\|\bar\W_{k+1\to 2} \diag{\b_1}\W_{1}\x\|^2] \nonumber \\
&=\theta^2\bE[\|\bar\W_{k+1\to 2} \W_{1} \x \|^2] + (\theta-\theta^2) \sum_{j=1}^{d_1} \bE[\| \bar\W_{k+1\to 2}(:,j) \|^2] {\color{green}\| \C^{\frac12} \W_{1}(j,:)\|^2}
\end{align}
The second equality follows from Lemma~\ref{lem:linear_norm}. Similarly, the first term on the right hand side of Equation~\eqref{eq:der2} can be expressed as:
\begin{align*}
\bE[\|\bar\W_{k+1\to 2} \W_{1} \x \|^2] &=\bE[\|\bar\W_{k+1\to 3} \diag{\b_2}\W_{2\to 1}\x \|^2]\\
&=\theta^2\bE[\|\bar\W_{k+1\to 3} \W_{2\to1} \x \|^2] + (\theta-\theta^2) \sum_{j=1}^{d_2} \bE[\| \bar\W_{k+1\to 3}(:,j) \|^2] {\color{green}\|\C^{\frac12} \W_{2\to 1}(j, :)\|^2}
\end{align*}
By recursive application of the above identity and plugging the result into Equation~\eqref{eq:der2}, we obtain: 
\begin{align}
\bE[\|\bar\W_{k+1 \to 1}\x\|^2] = \theta^{2k} \bE[\|\W_{k+1\to 1} \x \|^2] +(1-\theta) \sum_{i=1}^{k}\sum_{j=1}^{d_i}\theta^{2i-1} \bE[\| \bar\W_{k+1\to i+1}(:,j) \|^2] {\color{green}\|\C^{\frac12}\W_{i\to1}(j, :)\|^2}
\end{align}
Plugging back the above equality into Equation~\eqref{eq:der1}, we get 
\begin{align}\label{eq:der3}
L_\theta(\{\W_{i}\}) &=   \| \y \|^2 - 2 \bE \langle W_{k+1\to 1}\x,\y \rangle  + \bE[\|\W_{k+1\to 1} \x\|^2] +\frac{1-\theta}{\theta^{2k}} \sum_{i=1}^{k}\sum_{j=1}^{d_i}\theta^{2i-1}  \bE[\| \bar\W_{k+1\to i+1}(:,j) \|^2] {\color{green}\|\C^{\frac12} \W_{i\to1}(j,:)\|^2 }\nonumber\\
&= \bE_\x[\| \y - \W_{k+1\to 1}\x \|^2] +(1-\theta) \sum_{i=1}^{k}\sum_{j=1}^{d_i}\theta^{2(i-k)-1} \bE[\| \bar\W_{k+1\to i+1}(:,j) \|^2] {\color{green}\|\C^{\frac12} \W_{i\to1}(j,:)\|^2}.
\end{align}
It remains to calculate the terms of the form $\bE[\| \bar\W_{k+1 \to i+1}(:,j) \|^2]$ in the right hand side of Equation~\eqref{eq:der3}. We introduce the variable $x\sim\cN(0,1)$ so that we can use Lemma~\ref{lem:linear_norm} again:
\begin{align}\label{eq:der4}
\bE[\| \bar\W_{k+1\to i+1}(:,j) \|^2] &= \bE[\| \bar\W_{k+1\to i+2}\diag{\b_{i+1}}\W_{i+1}(:,j) x \|^2] \nonumber \\
&=\theta^2\bE[\| \bar\W_{k+1\to i+2}\W_{i+1}(:,j) \|^2]+(\theta-\theta^2) \bE   \sum_{l=1}^{d_{i+1}}{\|\bar\W_{k+1\to i+2}(:,l)\|^2 {\color{red}\W_{i+1}(l,j)^2}}.
\end{align}
The first term on the right hand side of Equation~\eqref{eq:der4} can be expanded as:
\begin{align*}
\bE[\| \bar\W_{k+1\to i+2}\W_{i+1}(:,j) \|^2] 
&=\bE[\| \bar\W_{k+1\to i+3}\diag{\b_{i+2}}\W_{i+2 \to i+1}(:,j) x \|^2]\\
&=\theta^2\bE[\| \bar\W_{k+1\to i+3}\W_{i+2\to i+1}(:,j) \|^2]+(\theta-\theta^2) \bE   \sum_{l=1}^{d_{i+2}}{\|\bar\W_{k+1\to i+3}(:,l)\|^2 {\color{red}\W_{i+2\to i+1}(l,j)^2}} \nonumber
\end{align*}
By recursive application of the above equality and plugging the results into Equation~\eqref{eq:der4}, we get
\begin{align*} 
\bE[\| \bar\W_{k+1\to i+1}(:,j) \|^2]  &=\theta^{2(k-i)}\|\W_{k+1\to i+1}(:,j) \|^2 + (1-\theta) \sum_{m=1}^{k-i}\theta^{2m-1} \bE   \sum_{l=1}^{d_{i+m}}{\|\bar\W_{k+1\to i+1+m}(:,l)\|^2 {\color{red}\W_{i+m\to i+1}(l,j)^2}}
\end{align*}
Plugging back the above identity into Equation~\eqref{eq:der3} we get
\begin{align*}
R(\{\W_{i}\}) &=  (1-\theta) \sum_{i=1}^{k}\sum_{j=1}^{d_i}\theta^{2(i-k)-1} {\color{green}\|\C^{\frac12} \W_{i\to1}(j,:)\|^2} \bE[\| \bar\W_{k+1\to i+1}(:,j) \|^2]\\
&=  \frac{1-\theta}{\theta} \sum_{i=1}^{k}\sum_{j=1}^{d_{i}} {\color{green}\| \C^{\frac12}\W_{i\to1}(j,:)\|^2} {\color{blue}\|\W_{k+1\to i+1}(:,j) \|^2} \\
&+(\frac{1-\theta}{\theta})^2 \sum_{i=1}^{k}\sum_{j=1}^{d_{i}} {\color{green}\|\C^{\frac12} \W_{i\to1}(j,:)\|^2}  \sum_{m=1}^{k-i}\theta^{2(i+m-k)} \bE   \sum_{l=1}^{d_{i+m}}{\color{red}\W_{i+m\to i+1}(l,j)^2}\|\bar\W_{k+1\to i+m+1}(:,l)\|^2 \\ 
&=  \frac{1-\theta}{\theta} \sum_{i=1}^{k}\sum_{j=1}^{d_{i}} {\color{green}\|\C^{\frac12}\W_{i\to1}(j,:)\|^2} {\color{blue}\|\W_{k+1\to i+1}(:,j) \|^2}\\
&+(\frac{1-\theta}{\theta})^2 \sum_{i=1}^{k}\sum_{j=1}^{d_i} {\color{green}\|\C^{\frac12}\W_{i\to1}(j,:)\|^2}  \sum_{m=1}^{k-i}  \sum_{l=1}^{d_{i+m}}{\color{red}\W_{i+m\to i+1}(l,j)^2}{\color{blue}\|\W_{k+1\to i+m+1}(:,l)\|^2} \\ 
&+(\frac{1-\theta}{\theta})^3 \sum_{i=1}^{k}\sum_{j=1}^{d_i} {\color{green}\|\C^{\frac12}\W_{i\to1}(j,:)\|^2}  \sum_{m=1}^{k-i}   \sum_{l=1}^{d_{i+m}}{\color{red}\W_{i+m\to i+1}(l,j)^2}   \left( \sum_{mm=1}^{k-i-m}\theta^{2(i+m+mm)}  \right. \\
& \left.  \sum_{ll=1}^{d_{i+m+mm}}{\color{red}\W_{i+m+mm\to i+m+1}(ll,l)^2} \bE\|\bar\W_{k+1\to i+1+m+mm}(:,ll)\|^2 \right)  \\
&= \cdots \\
&= \sum_{l=1}^k \lambda^l \sum_{\substack{ (j_l,\ldots , j_1) \\ \in {[k] \choose l} }}\sum_{\substack{(i_l,\ldots,i_1) \\ \in [d_{j_l}]\times\cdots \times [d_{j_1}]}} {\color{green} \| \C^{\frac12} \W_{j_1 \to 1}(i_1,:) \|^2 } {\color{red} \prod_{p=1\cdots l-1} \W_{j_{p+1}  \to j_{p}+1}(i_{p+1},i_p)^2 } {\color{blue} \| \W_{k+1 \to j_l+1}(:,i_l)\|^2},
\end{align*}
which completes the proof.
\end{proof}

\begin{lemma}\label{lem:spectral}[Properties of $R$ and $\Theta$] The following statements hold true:
\begin{enumerate}
\item All sub-regularizers, and hence the explicit regularizer, are re-scaling invariant.
\item The infimum in Equation~\eqref{eq:induced_reg} is always attained.
\item If $\C=\I$, then $\Theta(\M)$ is a \emph{spectral function}, i.e. if $\M$ and $\M'$ have the same singular values, then $\Theta(\M) = \Theta(\M')$.
\end{enumerate}
\end{lemma}
\begin{proof}[Proof of Lemma~\ref{lem:spectral}]
First, it is easy to see that the explicit regularizer and the sub-regularizers are all \emph{rescaling invariant}. For any sequence of scalars $\{\alpha_i\}$ such that such that $\prod_{i=1}^{k+1}\alpha_i=1$, let $\bar\W_i:=\alpha_i \W_i$ . Then it holds that:
\begin{align*}
R_l(\{ \bar\W_i\}) \!&=\!\!\! \sum_{\substack{ (j_l,\ldots , j_1) \\ \in {[k] \choose l} }}\!\!\sum_{\substack{(i_l,\ldots,i_1) \\ \in [d_{j_l}]\times\cdots \times [d_{j_1}]}} \!\!{\color{green} \| \prod_{q=1}^{j_1}\!\!\alpha_{q} \W_{j_1 \to 1}(i_1,:) \|^2 } \!\!{\color{red} \prod_{p\in [l-1]} \prod_{q=j_p + 1}^{j_{p+1}}\!\!\alpha_q^2\W_{j_{p+1}  \to j_{p}+1}(i_{p+1},i_p)^2 } {\color{blue} \| \!\!\prod_{q=j_l + 1}^{k+1}\!\!\alpha_q\W_{k+1 \to j_l+1}(:,i_l)\|^2} \\
&= \prod_{q=1}^{k+1}\alpha_q^2 \sum_{\substack{ (j_l,\ldots , j_1) \\ \in {[k] \choose l} }}\sum_{\substack{(i_l,\ldots,i_1) \\ \in [d_{j_l}]\times\cdots \times [d_{j_1}]}} {\color{green} \| \W_{j_1 \to 1}(i_1,:) \|^2 } {\color{red} \prod_{p\in[l-1]} \W_{j_{p+1}  \to j_{p}+1}(i_{p+1},i_p)^2 } {\color{blue} \| \W_{k+1 \to j_l+1}(:,i_l)\|^2} \\
&=R_l(\{\W_i\})
\end{align*}

 Therefore, without loss of generality, we can express the induced regularizer as follows:
\begin{align}\label{eq:induced_reg_redef}
\Theta(\M) := \infim{\substack{\W_{k+1}\cdots\W_1 = \M \\ \| \W_i \|_F \leq \| \M \|_F}}{R(\{ \W_i \})}
\end{align}
Note that $R(\{ \W_i \})$ is a continuous function and the feasible set $\mathcal{F}:=\{ (\W_i)_{i=1}^{k+1}: \W_{k+1}\cdots\W_1 = \M , \ \| \W_i \|_F \leq \| \M \|_F \}$ is compact. Hence, by Weierstrass extreme value theorem, the infimum is attained.

Now let $\U\in \R^{d_{k+1}\times d_{k+1}}$ and $\V\in \R^{d_{0} \times d_{0}}$ be a pair of rotation matrices, i.e. $\U^\top\U=\U\U^\top=\I$ and $\V^\top\V=\V\V^\top=\I$. When the data is isotropic, i.e. $\C=\I$, the following equalities hold
\begin{align*}
R(\{ \W_i\}) &= \sum_{l=1}^k \lambda^l \sum_{\substack{ (j_l,\ldots , j_1) \\ \in {[k] \choose l} }}\sum_{\substack{(i_l,\ldots,i_1) \\ \in [d_{j_l}]\times\cdots \times [d_{j_1}]}} {\color{green} \| \W_{j_1 \to 1}(i_1,:) \|^2 } {\color{red} \prod_{p=1\cdots l-1} \W_{j_{p+1}  \to j_{p}+1}(i_{p+1},i_p)^2 } {\color{blue} \| \W_{k+1 \to j_l+1}(:,i_l)\|^2} \\
&= \sum_{l=1}^k \lambda^l \sum_{\substack{ (j_l,\ldots , j_1) \\ \in {[k] \choose l} }}\sum_{\substack{(i_l,\ldots,i_1) \\ \in [d_{j_l}]\times\cdots \times [d_{j_1}]}} {\color{green} \| \W_{j_1 \to 1}(i_1,:)^\top \V \|^2 } {\color{red} \prod_{p=1\cdots l-1} \W_{j_{p+1}  \to j_{p}+1}(i_{p+1},i_p)^2 } {\color{blue} \| \U^\top \W_{k+1 \to j_l+1}(:,i_l)\|^2} \\
&= R(\U^\top\W_{k+1},\W_k,\cdots,\W_2,\W_1\V)
\end{align*}
That is, $R(\U^\top\W_{k+1},\W_k,\ldots,\W_2,\W_1\V) = R(\W_{k+1},\W_k,\ldots,\W_2,\W_1)$ for all rotation matrices $\U$ and $\V$. In particular, let $\U,\V$ be the left and right singular vectors of $\M$, i.e. $\M=\U\Sigma\V^\top$. To prove that $\Theta$ is a spectral function, we need to show that $\Theta(\M)=\Theta(\Sigma)$. Let $\{ \bar\W_i \},\{ \tilde\W_i \}$ be such that $\Theta(\M)=R(\{ \bar\W_i \}),\Theta(\Sigma)=R(\{ \tilde\W_i \})$. Note that such weight matrices always exist since the infimum is always attained. Then $$\Theta(\Sigma) = \Theta(\U^\top\M\V) \leq R(\U^\top\bar\W_{k+1},\bar\W_k,\ldots,\bar\W_2,\bar\W_1\V) = R(\bar\W_{k+1},\bar\W_k,\ldots,\bar\W_2,\bar\W_1) = \Theta(\M).$$ Similarly, we have that $\Theta(\M) \leq R(\U^\top\tilde\W_{k+1},\tilde\W_k,\ldots,\tilde\W_2,\bar\W_1\V) = R(\tilde\W_{k+1},\bar\W_k,\ldots,\tilde\W_2,\tilde\W_1) = \Theta(\Sigma)$, which completes the proof.
\end{proof}

\subsection{The induced regularizer and its convex envelope}
\begin{proof}[Proof of Theorem~\ref{thm:envelope}]
By Lemma~\ref{lem:reg_lb}, for any architecture, any dropout rate, and any set of weights $\{ \W_i \}$ that implements a network map $\W_{k+1\to 1}$, the explicit regularizer is lower bounded by the effective regularization parameter times the product of the squared nuclear norm of the network map and the principal squared root of the second moment of $\x$, i.e.  $R(\{ \W_i \}) \geq \nu_{\{ d_i \}}\| \W_{k+1\to 1}  \C^{\frac12}  \|_*^2$. Consequently, the induced regularizer can also be lowerbounded as $\Theta(\M)\geq \nu_{\{ d_i \}}\| \M  \C^{\frac12}  \|_*^2$. On the other hand, Lemma~\ref{lem:fenchel} establishes that $\Theta^{**}(\M)\leq \nu_{\{ d_i \}}\| \M \C^{\frac12}  \|_*^2$ holds for any network map $\M$. Putting these two inequalities together, we arrive at $$\Theta^{**}(\M)\leq \nu_{\{ d_i \}}\| \M \C^{\frac12}  \|_*^2 \leq \Theta(\M).$$ Since $\Theta^{**}(\M)$ is the largest convex underestimator of $\Theta(\M)$, and the squared nuclear norm is a convex function, we conclude that $\Theta^{**}(\M) = \nu_{\{ d_i \}}\| \M \|_*^2$.
\end{proof}

Despite the complex form of the explicit regularizer given in Proposition~\ref{prop:reg}, we can show that it is always lower bounded by \emph{effective regularization parameter} times $\| \M \C^\frac12\|_*^2$. This result is given by Lemma~\ref{lem:reg_lb}.
\begin{proof}[Proof of Lemma~\ref{lem:reg_lb}]
Recall that the explicit regularizer $R(\{ \W_i \})$ is composed of $k$ \emph{sub-regularizers} $$R(\{ \W_i \}) = R_1(\{ \W_i \}) + R_2(\{ \W_i \}) + \cdots + R_k(\{ \W_i \}).$$ The $l$-th sub-regularizer $R_l(\{\W_i\})$ can be written in the form of:
$$R_l(\{ \W_i \})=\lambda^l \sum_{( j_l, \ldots, j_1 ) \in {[k] \choose l}}R_{\{ j_l, \ldots, j_1 \}}(\{ \W_i \})$$
where $$R_{\{ j_l, \ldots, j_1 \}}(\{ \W_i \}) := {\color{green} \| \C^{\frac12} \W_{j_1 \to 1}(i_1,:) \|^2 } {\color{red} \prod_{p=1\cdots l-1} \W_{j_{p+1}  \to j_{p}+1}(i_{p+1},i_p)^2 } {\color{blue} \| \W_{k+1 \to j_l+1}(:,i_l)\|^2}.$$ The following set of equalities hold true:
\begin{align*}
&R_{\{ j_l, \ldots, j_1 \}}(\{ \W_i \})=\sum_{i_l,\cdots,i_1  } {\color{blue}\| \W_{k+1\to j_l+1}(:,i_l) \|^2}  {\color{red}\W_{j_{l}\to j_{l-1}+1}(i_l,i_{l-1})^2 \cdots \W_{j_{2}\to j_{1}+1}(i_2,i_{1})^2}  {\color{green}\|  \C^{\frac12} \W_{j_1 \to 1}(i_1,:) \|^2} \\
&\geq \frac{1}{\prod_{i\in[l]} d_{j_i}} \left( \sum_{i_l,\ldots,i_1} {\color{blue}\| \W_{k+1\to j_l+1}(:,i_l) \|} {\color{red} | \W_{j_{l}\to j_{l-1}+1}(i_l,i_{l-1})| \cdots |\W_{j_{2}\to j_{1}+1}(i_2,i_{1})|}  {\color{green}\|  \C^{\frac12} \W_{j_1 \to 1}(i_1,:) \|} \right)^2\\
&= \frac{1}{\prod_{i\in[l]} d_{j_i}} \left( \sum_{i_l,\ldots,i_1} \| {\color{blue}\W_{k+1\to j_l+1}(:,i_l)} {\color{red}\W_{j_{l}\to j_{l-1}+1}(i_l,i_{l-1}) \cdots \W_{j_{2}\to j_{1}+1}(i_2,i_{1})} {\color{green}\W_{j_1 \to 1}(i_1,:)^\top  \C^{\frac12}} \|_* \right)^2\\
&\geq \frac{1}{\prod_{i\in[l]} d_{j_i}}   \| \sum_{i_l,\ldots,i_1} {\color{blue}\W_{k+1\to j_l+1}(:,i_l)} {\color{red}\W_{j_{l}\to j_{l-1}+1}(i_l,i_{l-1}) \cdots \W_{j_{2}\to j_{1}+1}(i_2,i_{1})} {\color{green}\W_{j_1 \to 1}(i_1,:)^\top  \C^{\frac12} } \|_*^2\\
&= \frac{1}{\prod_{i\in[l]} d_{j_i}}   \| \sum_{i_l,i_1} {\color{blue}\W_{k+1\to j_l+1}(:,i_l)}  \left( \sum_{i_{l-1},\ldots,i_2} {\color{red}\W_{j_{l}\to j_{l-1}+1}(i_l,i_{l-1}) \cdots \W_{j_{2}\to j_{1}+1}(i_2,i_{1})} \right) {\color{green}\W_{j_1 \to 1}(i_1,:)^\top \C^{\frac12}}  \|_*^2\\
&= \frac{1}{\prod_{i\in[l]} d_{j_i}}   \| \sum_{i_l,i_1} {\color{blue}\W_{k+1\to j_l}(:,i_l)} {\color{red}\W_{j_{l}-1\to j_{1}}(i_l,i_{1})} {\color{green}\W_{j_1 \to 1}(i_1,:)^\top \C^{\frac12} } \|_*^2\\
&= \frac{1}{\prod_{i\in[l]} d_{j_i}}   \| \W_{k+1}\cdots \W_1 \C^{\frac12}  \|_*^2
\end{align*}
where the first inequality follows due to the Cauchy-Schwartz inequality, and the second inequality follows from the triangle inequality for the matrix norms. The inequality holds with equality if and only if all the summands inside the summation are equal to each other, and sum up to $\frac{\| \W_{k+1\to 1}  \C^{\frac12} \|_*}{\prod_{i \in [l]}d_{j_i}} $, i.e. when $${\color{blue}\| \W_{k+1\to j_l+1}(:,i_l) \|} {\color{red} | \W_{j_{l}\to j_{l-1}+1}(i_l,i_{l-1})| \cdots |\W_{j_{2}\to j_{1}+1}(i_2,i_{1})|} {\color{green}\|  \C^{\frac12} \W_{j_1 \to 1}(i_1,:) \|} = \frac{1}{\prod_{i \in [l]}d_{j_i}} \| \W_{k+1\to 1}  \C^{\frac12} \|_*$$ for all $(i_l,\ldots,i_1) \in [d_{j_l}] \times \cdots \times [d_{j_1}]$. This lowerbound holds for all $l\in [k]$, and for all $(j_l,\ldots,j_1)\in {[k] \choose l}$. Thus, we get the following lowerbound on the regularizer:
\begin{align*}
R(\{ \W_i \}) \geq \sum_{l\in [k]}\lambda^l\sum_{ (j_l,\ldots , j_1) \in {[k] \choose l}}\frac{1}{\prod_{i\in [l]} d_{j_i}} \| \W_{k+1\to 1} \C^{\frac12}  \|_*^2 = \nu_{\{\d_i\}} \| \W_{k+1\to 1}  \C^{\frac12}  \|_*^2
\end{align*}
which completes the proof. 
\end{proof}

Not only $\nu_{\{ d_i \}}\| \M\C^\frac12 \|_*^2$ is a lowerbound for the induced regularizer, but also is an upperbound for its convex envelope. We prove this result in  Lemma~\ref{lem:fenchel}.
\begin{proof}[Proof of Lemma~\ref{lem:fenchel}]
The induced regularizer is non-negative. Hence, the domain of the Fenchel dual of the induced regularizer is the whole $\R^{d_{k+1}\times d_0}$. The Fenchel dual of the induced regularizer $\Theta(\cdot)$ is given by:
\begin{align} \label{eq:fenchel_deep}
\Theta^*(\M) &= \maxim{\P}{\langle \M,\P \rangle - \Theta(\P)} \nonumber\\
&= \maxim{\P}{\langle \M,\P \rangle - \minim{\substack{\{ \W_i \} \\ \W_{k+1\to 1} = \P}}{R(\{ \W_i \})}}\nonumber\\
&= \maxim{\{\W_i\}}{\langle \M,\W_{k+1 \to 1} \rangle - R(\{ \W_i \})}.
\end{align}
Define $\Phi(\{ \W_i \}) := \langle \M,\W_{k+1 \to 1} \rangle - R(\{ \W_i \})$ as the objective in the right hand side of Equation~\eqref{eq:fenchel_deep}. Due to the complicated products of the norms of the weights in the regularizer, maximizing $\Phi$ with respect to $\{ \W_i \}$ is a daunting task. Here, we find a lower bound on this maximum value. Let $\W^\alpha_{k+1}:=\alpha\u_1\1_{d_{k}}^\top$ and $\W^\alpha_1:=\1_{d_1}\v_1^\top \C^{-\frac12}$, where $(\u_1,\v_1)$ is the top singular vectors of $\M \C^{-\frac12}$, and $\1_d$ is the $d$-dimensional vector of all $1$s. Furthermore, let $\W^\alpha_i := \1_{d_{i}}\1_{d_{i-1}}^\top$, for all $i \in \{ 2,\ldots,k \}$. 
Note that $$\Theta^*(\M) = \maxim{\{\W_i\}}{\Phi(\{ \W_i\})} \geq \maxim{\alpha}{\Phi(\{ \W^\alpha_i \})}.$$ We now simplify $\Phi(\{ \W^\alpha_i \})$. First, the following equalities hold for the $\langle \M, \W^\alpha_{k+1 \to 1}\rangle$:
\begin{align*}
\langle \M, \W^\alpha_{k+1 \to 1}\rangle &=\sum_{(i_{k+1},\ldots,i_1) \in [d_{k+1}]\times \cdots \times [d_1]} \langle \M, {\color{blue}  \W^\alpha_{k+1}(:,i_{k}) } {\color{red} \prod_{j=\{k-1,\ldots, 1\}} \W^\alpha_{j+1}(i_{j+1},i_j) } {\color{green}  \W^\alpha_{1}(i_1,:)^\top \rangle} \\
&=\sum_{(i_{k+1},\ldots,i_1) \in [d_{k+1}]\times \cdots \times [d_1]}  {\color{blue}  \W^\alpha_{k+1}(:,i_{k})^\top }\M {\color{green}  \W^\alpha_{1}(i_1,:) } \\
&=\sum_{(i_{k+1},\ldots,i_1) \in [d_{k+1}]\times \cdots \times [d_1]}  \alpha \u_1^\top \M  \C^{-\frac12} \v_1 \\
&= \sum_{(i_{k+1},\ldots,i_1) \in [d_{k+1}]\times \cdots \times [d_1]} \alpha \| \M  \C^{-\frac12} \|_2 \\
&= \alpha \|  \M \C^{-\frac12} \|_2 \prod_{j\in[k]}d_j =: \alpha \| \M  \C^{-\frac12}\|_2 D.
\end{align*}
The following terms show up in the expansion of the regularizer:
\begin{align*}
{\color{green} \W^\alpha_{j_1 \to 1}(i_1,:)^\top } &= \W^\alpha_{j_1}(i_1,:)\W^\alpha_{j_1 -1} \cdots \W^\alpha_2 \W^\alpha_1 = \1_{d_{j_1-1}}^\top \1_{d_{j_1-1}} \1_{d_{j_1-2}}^\top \cdots \1_{d_2}\1_{d_1}^\top \1_{d_1}\v_1^\top = \prod_{i\in[j_1 -1]}d_i \v_1^\top \C^{-\frac12}  \\
{\color{red} \W^\alpha_{j_{p+1}  \to j_{p}+1}(i_{p+1},i_p) }  &= \W^\alpha_{j_{p+1}}(i_{p+1},:)\W^\alpha_{j_{p+1}-1} \cdots \W^\alpha_{j_p + 2} \W^\alpha_{j_p + 1}(:,i_p)\\
&= \1_{d_{j_{p+1}-1}}^\top \1_{d_{j_{p+1}-1}} \1_{d_{j_{p+1}-2}}^\top \cdots \1_{d_{j_p + 2}}\1_{d_{j_p + 1}}^\top \1_{d_{j_p + 1}} = \prod_{i\in \{ j_p + 1, \cdots, j_{p+1}-1 \} }d_i  \\
{\color{blue} \W^\alpha_{k+1 \to j_l+1}(:,i_l) } &=  \alpha\W^\alpha_{k+1}\W^\alpha_{k} \cdots \W^\alpha_{j_l +2} \W^\alpha_{j_l +1}(:,i_l) = \alpha \u_1 \1_{d_{k}}^\top \1_{d_{k}} \1_{d_{k-1}}^\top \cdots \1_{d_{j_l+2}}\1_{d_{j_l + 1}}^\top \1_{d_{j_l + 1}} = \alpha\prod_{i\in \{j_l + 1, \cdots, k\}}d_i \u_1
\end{align*}
With the above equalities, the explicit regularizer reduces to:
\begin{align*}
R(\{\W^\alpha_i\}) &= \sum_{l=1}^k \lambda^l \sum_{\substack{ (j_l,\ldots , j_1) \\ \in {[k] \choose l} }}\sum_{\substack{(i_l,\ldots,i_1) \\ \in [d_{j_l}]\times\cdots \times [d_{j_1}]}} {\color{green} \|  \C^{\frac12} \W^\alpha_{j_1 \to 1}(i_1,:) \|^2 } {\color{red} \prod_{p=1\cdots l-1} \W^\alpha_{j_{p+1}  \to j_{p}+1}(i_{p+1},i_p)^2 } {\color{blue} \| \W^\alpha_{k+1 \to j_l+1}(:,i_l)\|^2}\\
&= \sum_{l=1}^k \lambda^l \sum_{\substack{ (j_l,\ldots , j_1) \\ \in {[k] \choose l} }}\sum_{\substack{(i_l,\ldots,i_1) \\ \in [d_{j_l}]\times\cdots \times [d_{j_1}]}} {\color{green} \|  \C^{\frac12} \C^{-\frac12} \v_1 \prod_{i\in [j_1 - 1]} d_i \|^2 } {\color{red} \prod_{p=1\cdots l-1} \prod_{i \in \{ j_p+1,\cdots,j_{p+1}-1 \}}d_i^2  } {\color{blue} \| \alpha \u_1 \prod_{i \in \{ j_l +1 ,\cdots, k \}}d_i \|^2}\\
&= \sum_{l=1}^k \lambda^l \sum_{\substack{ (j_l,\ldots , j_1) \\ \in {[k] \choose l} }}\sum_{\substack{(i_l,\ldots,i_1) \\ \in [d_{j_l}]\times\cdots \times [d_{j_1}]}} {\color{green} \prod_{i \in [j_1 - 1]}d_i^2 } {\color{red} \prod_{p=1\cdots l-1} \prod_{i \in \{ j_p+1,\cdots,j_{p+1}-1 \}}d_i^2 } {\color{blue}\alpha^2 \prod_{i\in \{ j_l+1,\cdots, k \}}d_i^2}\\
&= \alpha^2 \sum_{l=1}^k \lambda^l \sum_{ (j_l,\ldots , j_1)  \in {[k] \choose l} }\sum_{(i_l,\ldots,i_1) \in [d_{j_l}]\times\cdots \times [d_{j_1}]}   \frac{\prod_{i \in [k]}d_i^2}{\prod_{i \in [l]}d_{j_i}^2}=: \alpha^2 \rho 
\end{align*}
Plugging back the above equalities into the definition of $\Phi$, we arrive at $\Phi(\{ \W^\alpha_i \}) = \alpha \|  \M \C^{-\frac12} \|_2 D - \alpha^2 \rho$. The maximum of $\Phi(\{ \W^\alpha_i \})$ with respect to $\alpha$ is achieved when $\alpha^*  =  \frac{\| \M  \C^{-\frac12} \|_2 D}{2\rho}$, in which case we have $$\Theta^*(\M) \geq \Phi(\{ \W^{\alpha^*}_i \}) = \frac{D^2}{4\rho}\| \M \C^{-\frac12}  \|_2^2 =: \Psi(\M).$$
Since Fenchel dual is order reversing, we get 
\begin{align*}
\Theta^{**}(\M) &\leq \Psi^*(\M) \nonumber\\
&= \frac{\rho}{D^2} \| \M  \C^{\frac12} \|_*^2 \\
&= \frac{\sum_{l=1}^k \lambda^l \sum_{ (j_l,\ldots , j_1)  \in {[k] \choose l} }\sum_{(i_l,\ldots,i_1) \in [d_{j_l}]\times\cdots \times [d_{j_1}]}   \frac{\prod_{i \in [k]}d_i^2}{\prod_{i \in [l]}d_{j_i}^2}}{\prod_{j\in[k]}d_j^2} \| \M  \C^{\frac12} \|_*^2 \\
&= \sum_{l=1}^k \lambda^l \sum_{ (j_l,\ldots , j_1)  \in {[k] \choose l} }\sum_{(i_l,\ldots,i_1) \in [d_{j_l}]\times\cdots \times [d_{j_1}]}   \frac{1}{\prod_{i \in [l]}d_{j_i}^2} \| \M \C^{\frac12} \|_*^2\\
&= \sum_{l=1}^k \lambda^l \sum_{ (j_l,\ldots , j_1)  \in {[k] \choose l} } \frac{1}{\prod_{i \in [l]}d_{j_i}} \| \M \C^{\frac12} \|_*^2\\
&= \nu_{\{d_i \}} \| \M  \C^{\frac12} \|_*^2
\end{align*}
where the first equality follows from the fact that if $f(\M)=\beta \| \M \A \|^2$  and  $\A \succ \0$ then $f^*(\M)= \frac1{4\beta} \|\M\A^{-1}  \|_*^2$. This result is standard in the literature, but we prove it here for completeness. Note that
\begin{align*}
\langle \Y, \M \rangle - \beta \| \Y \A \|^2 &= \langle \Y\A, \M\A^{-1} \rangle - \beta \| \Y \A \|^2 \\
 &\leq \| \Y \A \| \| \M \A^{-1}\|_* - \beta \| \Y \A \|^2
\end{align*} 
where the inequality is due to Holder's identity. The right hand side above is a quadratic in $\| \Y\A \|$ and is maximized when $\| \Y\A \| = \frac1{2\beta} \| \M\A^{-1} \|_*$, in which case we have
$$f^*(\M) = \sup_{\Y} \langle \Y, \M \rangle - \beta \| \Y \A \|^2 = \frac1{2\beta} \| \M\A^{-1} \|_* \| \M \A^{-1}\|_* - \beta (\frac1{2\beta} \| \M\A^{-1} \|_*)^2 = \frac1{4\beta} \| \M\A^{-1} \|_*^2.$$
\end{proof}

\subsection{Characterization of the global optima of the dropout objective}
\begin{proof}[Proof of Proposition~\ref{prop:equalizable}]
When the network map has rank equal to one, it can be expressed as $\u\v^\top$, where $\u\in \R^{d_{k+1}}$ and $\v\in \R^{d_0}$. We show that for any architecture $\{ d_i \}$ and any network mapping $\u\v^\top \in \R^{d_{k+1} \times d_0}$, it is always possible to represent $\u\v^\top=\W_{k+1}\cdots\W_1$ such that the resulting network is equalized. One such factorization is when $\W_1=\frac{\1_{d_1} \v^\top }{\sqrt d_1}$, $\W_{k+1} = \frac{\u\1_{d_k}^\top}{\sqrt{d_k}}$, and $\W_i = \frac{\1_{d_i}\1_{d_{i-1}}^\top}{\sqrt{d_i d_{i-1}}}$ for $i\in \{2,\ldots, k \}$. For these weight parameters, we have that
\begin{align*}
{\color{green} \W_{j_1 \to 1}(i_1,:)^\top } &= \W_{j_1}(i_1,:)^\top\W_{j_1 -1} \cdots \W_2 \W_1 = \frac{\1_{d_{j_1-1}}^\top}{\sqrt{d_{j_1} d_{j_1-1}}} \frac{\1_{d_{j_1-1}} \1_{d_{j_1-2}}^\top}{\sqrt{d_{j_1-1} d_{j_1-2}}}  \cdots \frac{\1_{d_2}\1_{d_1}^\top}{\sqrt{d_2 d_1}}  \frac{\1_{d_1} \v^\top }{\sqrt{d_1}} = \frac{\v^\top }{\sqrt{d_{j_1}}} \\
{\color{red} \W_{j_{p+1}  \to j_{p}+1}(i_{p+1},i_p) }  &= \W_{j_{p+1}}(i_{p+1},:)^\top \W_{j_{p+1}-1} \cdots \W_{j_p + 2} \W_{j_p + 1}(:,i_p)\\
&= \frac{\1_{d_{j_{p+1}-1}}^\top}{\sqrt{d_{j_{p+1}} d_{j_{p+1}-1}}} \frac{ \1_{d_{j_{p+1}-1}} \1_{d_{j_{p+1}-2}}^\top}{\sqrt{d_{j_{p+1}-1} d_{j_{p+1}-2}}} \cdots \frac{\1_{d_{j_p + 2}}\1_{d_{j_p + 1}}^\top}{\sqrt{d_{j_p+2} d_{j_p+1}}} \frac{\1_{d_{j_p + 1}}}{\sqrt{d_{j_p+1} d_{j_p}}}  = \frac1{\sqrt{d_{j_{p+1}} d_{j_p}}}  \\
{\color{blue} \W_{k+1 \to j_l+1}(:,i_l) } &= \W_{k+1}\W_{k} \cdots \W_{j_l +2} \W_{j_l +1}(:,i_l)\\
& =  \frac{\u\1_{d_k}^\top}{\sqrt{d_k}} \frac{\1_{d_k}\1_{d_{k-1}}^\top}{\sqrt{d_k d_{k-1}}}  \cdots \frac{\1_{d_{j_l+2}}\1_{d_{j_l + 1}}^\top}{\sqrt{d_{j_l+2} d_{j_l + 1}}} \frac{\1_{d_{j_l + 1}}}{\sqrt{d_{j_l+1} d_{j_l}}} = \frac{\u}{\sqrt{ d_{j_l}}}
\end{align*}
With the above equalities, the regularizer reduces to:
\begin{align*}
R(\{\W_i\}) &= \sum_{l=1}^k \lambda^l \sum_{\substack{ (j_l,\ldots , j_1) \\ \in {[k] \choose l} }}\sum_{\substack{(i_l,\ldots,i_1) \\ \in [d_{j_l}]\times\cdots \times [d_{j_1}]}} {\color{green} \|  \C^{\frac12} \W_{j_1 \to 1}(i_1,:) \|^2 } {\color{red} \prod_{p=1\cdots l-1} \W_{j_{p+1}  \to j_{p}+1}(i_{p+1},i_p)^2 } {\color{blue} \| \W_{k+1 \to j_l+1}(:,i_l)\|^2}\\
 &= \sum_{l=1}^k \lambda^l \sum_{\substack{ (j_l,\ldots , j_1) \\ \in {[k] \choose l} }}\sum_{\substack{(i_l,\ldots,i_1) \\ \in [d_{j_l}]\times\cdots \times [d_{j_1}]}} {\color{green} \|  \C^{\frac12} \frac{\v}{\sqrt{d_{j_1}}} \|^2 } {\color{red} \prod_{p=1\cdots l-1} \frac1{ d_{j_{p+1}} d_{j_p} } } {\color{blue} \| \frac{\u}{\sqrt{ d_{j_l}}} \|^2}\\
 &= \sum_{l=1}^k \lambda^l \sum_{\substack{ (j_l,\ldots , j_1) \\ \in {[k] \choose l} }}\sum_{\substack{(i_l,\ldots,i_1) \\ \in [d_{j_l}]\times\cdots \times [d_{j_1}]}}  \frac{\|  \C^{\frac12} \v \|^2 \| \u \|^2}{\prod_{p \in [l]} d_{j_p}^2}     \\
 &= \sum_{l=1}^k \sum_{\substack{ (j_l,\ldots , j_1) \\ \in {[k] \choose l} }}  \frac{ \lambda^l}{\prod_{p \in [l]} d_{j_p}} \|  \u\v^\top\C^{\frac12} \|_*^2 = \nu_{\{ d_i \}} \|\u\v^\top\C^{\frac12} \|_*^2
\end{align*}
where we used the fact that $\|  \u\|  \|\C^{\frac12} \v \| = \|  \u\v^\top\C^{\frac12} \|_*$.
Moreover, note that the network specified by the above weight matrices is equalized, since $${\color{green} |\alpha_{j_1,i_1}| } {\color{red} \prod_{p=1\cdots l-1} |\beta_p| } {\color{blue} |\gamma_{j_l,i_l}|} = \sqrt{  \frac{\|  \u\v^\top\C^{\frac12} \|_*^2}{\prod_{p \in [l]} d_{j_p}^2}} =  \frac{1}{\prod_{p \in [l]} d_{j_p}} \|  \u\v^\top \C^{\frac12} \|_*.$$
\end{proof}

\begin{lemma}\label{lem:inverse}
For any integer $r$, and for any $\nu \in \R_+$, it holds that $$(\I_r + \nu \1_r\1_r^\top)^{-1}=\I_r - \frac{\nu}{1+ r \nu}\1_r\1_r^\top.$$
\end{lemma}
Lemma~\ref{lem:inverse} is an instance of the Woodbury's  matrix identity. Here, we include a proof for completeness. 
\begin{proof}[Proof of Lemma~\ref{lem:inverse}]
The proof simply follows from the following set of equations.
\begin{align*}
(\I_r + \nu \1_r\1_r^\top)(\I_r - \frac{\nu}{1+ r \nu}\1_r\1_r^\top) &=\I_r + \nu\1_r\1_r^\top  - \frac{\nu}{1+ r\nu}\1_r\1_r^\top  - \frac{\nu^2}{1+ r\nu}\1_r\1_r^\top \1_r\1_r^\top\\
&=\I_r + \left( \nu  - \frac{\nu}{1+ r\nu}  - \frac{\nu^2 r}{1+ r\nu} \right) \1_r\1_r^\top=\I_r
\end{align*}
\end{proof}

\begin{lemma}\label{lem:global}
Consider the following optimization problem where the induced regularizer in Problem~\ref{eq:optim} is replaced with its convex envelope:
\begin{equation}\label{eq:optim_convex}
\minim{\W\in\R^{d_{k+1}\times d_0}}{\bE[\| \y-\W\x \|^2] + \Theta^{**}(\W)}, \quad \rank{\W} \leq \min_{i\in[k+1]}d_i =: r
\end{equation}
Define the ``model'' $\bar\M:=\C_{\y\x}\C^{-\frac12}$. The global optimum of problem~\ref{eq:optim_convex} is given as $\M^* = \cS_{\alpha_\rho}(\bar\M)\C^{-\frac12}$, where $\alpha_\rho:=\frac{{\nu_{\{ d_i \}}}\sum_{j=1}^{\rho}\sigma_j(\bar\M)}{1+\rho{\nu_{\{ d_i \}}}}$, and $\rho \in [\min\{ r, \rank{\bar\M}\}]$ is the largest integer such that for all $i\in [\rho]$, it holds that $\sigma_i(\bar\M) > \alpha_\rho$.
\end{lemma}

\begin{proof}[Proof of Lemma~\ref{lem:global}]
Denote the objective in the optimization problem~\eqref{eq:optim_convex} as $\cE_{\nu_{\{ d_i \}}}(\W):=\bE[\| \y-\W\x \|^2] + {\nu_{\{ d_i \}}} \| \W\C^\frac12 \|_*^2$. Let $\C_\y:=\bE[\y\y^\top]$ and $\C_{\x\y}:=\bE[\x\y^\top]$. Note that 
\begin{align*}
\minim{\rank{\W}\leq r}{\cE_{\nu_{\{ d_i \}}}(\W)}&= \minim{\rank{\W}\leq r}{\bE[\| \y\|^2] + \bE[\|\W\x \|^2] - 2\bE[\langle \y,\W\x \rangle] + {\nu_{\{ d_i \}}} \| \W\C^\frac12 \|_*^2}\\
&\equiv \minim{\rank{\W}\leq r}{ \trace{\bE[\W\x \x^\top \W^\top]} - 2\trace{\bE[\W\x\y^\top]} + {\nu_{\{ d_i \}}} \| \W\C^\frac12 \|_*^2}\\
&= \minim{\rank{\W}\leq r}{ \trace{\W\C\W^\top} - 2\trace{\W\C_{\x\y}} + {\nu_{\{ d_i \}}} \| \W\C^\frac12 \|_*^2}
\end{align*}
Make the change of variable $\bar\W\gets \W\C^\frac12$ and denote $\bar\M:=\C_{\y\x}\C^{-\frac12}$, the goal is to solve the following problem 
\begin{equation}\label{eq:sym_equalized}
 \minim{\rank{\W}\leq r}{ \trace{\bar\W\bar\W^\top} - 2\langle \bar\W, \bar\M\rangle + {\nu_{\{ d_i \}}} \| \bar\W \|_*^2}\equiv \minim{\rank{\bar\W}\leq r}{ \| \bar\M - \bar\W \|_F^2 + {\nu_{\{ d_i \}}} \| \bar\W \|_*^2}
\end{equation}
If $\bar\W$ is a solution to the above problem, then a solution to the original problem in Equation~\eqref{eq:optim_convex} is given as $\bar\W\C^{-\frac12}$. Following~\cite{Cavazza2017analysis,mianjy2018implicit}, we show that the global optimum of Problem~\ref{eq:sym_equalized} is given in terms of an appropriate shrinkage-thresholding on the spectrum of $\bar\M$. 
Define $r':=\max\{\rank{\bar\M},r\}$. Let $\bar\M=\U_{\bar\M} \Sigma_{\bar\M} \V_{\bar\M}^\top$ and $\bar\W =\U_{\bar\W} \Sigma_{\bar\W} \V_{\bar\W}^\top$ be rank-$r'$ SVDs of $\bar\M$ and $\bar\W$ respectively, such that $\sigma_i(\bar\M)\geq \sigma_{i+1}(\bar\M)$ and $\sigma_{i}(\bar\W) \geq \sigma_{i+1}(\bar\W)$ for all $i\in[r'-1]$. Rewriting objective of Problem~\ref{eq:sym_equalized} in terms of these decompositions gives: 
\begin{align*}
&\| \bar\M - \bar\W \|_F^2 + {\nu_{\{ d_i \}}} \|\bar\W\|_*^2 =\| \U_{\bar\M} \Sigma_{\bar\M} \V_{\bar\M}^\top -  \U_{\bar\W} \Sigma_{\bar\W} \V_{\bar\W}^\top \|_F^2 + {\nu_{\{ d_i \}}} \| \U_{\bar\W} \Sigma_{\bar\W} \V_{\bar\W}^\top \|_*^2\\
&=\| \Sigma_{\bar\M}  -  \U' \Sigma_{\bar\W} \V'^\top \|_F^2 + {\nu_{\{ d_i \}}}\|\Sigma_{\bar\W}\|_*^2\\
&=\| \Sigma_{\bar\M} \|_F^2 + \| \Sigma_{\bar\W} \|_F^2 - 2\langle \Sigma_{\bar\M}, \U'\Sigma_{\bar\W} \V'^\top \rangle + {\nu_{\{ d_i \}}}\|\Sigma_{\bar\W}\|_*^2
\end{align*}
where $\U'=\U_{\bar\M}^\top \U_{\bar\W}$ and $\V'=\V_{\bar\M}^\top \V_{\bar\W}$. By Von Neumann's trace inequality, for a fixed $\Sigma_{\bar\W}$ we have that $$\langle \Sigma_{\bar\M}, \U'\Sigma_{\bar\W} \V'^\top \rangle \leq  \sum_{i=1}^{r'}{\sigma_i(\bar\M)\sigma_i(\bar\W)},$$ where the equality is achieved when $\U_{\bar\M}=\U_{\bar\W}$ and $\V_{\bar\M}=\V_{\bar\W}$. Hence, problem~\ref{eq:sym_equalized} is reduced to 
\begin{align*}
&\minim{\substack{\|\Sigma_{\bar\W}\|_0\leq r,\\ \Sigma_{\bar\W} \geq 0}}{\| \Sigma_{\bar\M}  -  \Sigma_{\bar\W} \|_F^2 + {\nu_{\{ d_i \}}} \left(\tr(\Sigma_{\bar\W})\right)^2} = \minim{\bar\sigma \in \R^r_+}{\sum_{i=1}^{r}{\left(\lambda_i(\bar\M) - \bar\sigma_i\right)^2}  + {\nu_{\{ d_i \}}} \left( \sum_{i=1}^{r}{\bar\sigma_i} \right)^2}
\end{align*}
The Lagrangian is given by
 \begin{align*}
 L(\bar\sigma,\alpha)&=\sum_{i=1}^{r}{\left(\lambda_i(\bar\M) - \bar\sigma_i\right)^2} + {\nu_{\{ d_i \}}} \left( \sum_{i=1}^{r}{\bar\sigma_i} \right)^2 - \sum_{i=1}^{r}{\alpha_i\bar\sigma_i}
 \end{align*}
  The KKT conditions ensures that at the optima it holds for all $i \in [r]$ that 
\begin{align*}
\bar\sigma_i \geq 0 , \ \alpha_i \geq 0 , \ \bar\sigma_i\alpha_i = 0 , \quad  2(\bar\sigma_i-\sigma_i(\bar\M)) + 2{\nu_{\{ d_i \}}} \sum_{j=1}^{r}{\bar\sigma_j} - \alpha_i =0
\end{align*}

Let $\rho = |{i: \bar\sigma_i > 0}|\leq r$ be the number of nonzero $\bar\sigma_i$, i.e. rank of the global optimum $\bar\W$. For $i \in [\rho]$, we have $\alpha_i = 0$. Therefore, we have that: 

\begin{align*}
\bar\sigma_i + {\nu_{\{ d_i \}}} \sum_{j=1}^{r}{\bar\sigma_j} = \sigma_i(\bar\M) &\implies (\I_\rho + {\nu_{\{ d_i \}}}\1_\rho\1_\rho^\top)\bar\sigma_{1:\rho} = \sigma_{1:\rho}(\bar\M) \\
&\implies \bar\sigma_{1:\rho} = (\I_\rho - \frac{{\nu_{\{ d_i \}}}}{1+\rho {\nu_{\{ d_i \}}}}\1_\rho\1_\rho^\top)\sigma_{1:\rho}(\bar\M)  = \sigma_{1:\rho}(\bar\M)-\frac{{\nu_{\{ d_i \}}} \rho \kappa_\rho}{1+\rho {\nu_{\{ d_i \}}}}\1_\rho
\end{align*}
where $ \kappa_j:=\frac1{j}\sum_{i=1}^{j}{\sigma_i(\bar\M)}$. The equation above tell us that  for $i\in[\rho]$, the singular values of $\bar\W$ are just a shrinkage of the singular values of $\bar\M$. In particular, it means that $\rho \leq \rank{\bar\M}$. Therefore, { without loss of generality, we assume that $r \leq \rank{\bar\M}$.} Also, since $\bar\sigma_i > 0$ for all $i\in [\rho]$, it holds that $\sigma_i(\bar\M) > \frac{{\nu_{\{ d_i \}}}\rho \kappa_\rho}{1+\rho{\nu_{\{ d_i \}}}}$ for all $i\in [\rho]$.
For $i \in \{ \rho+1,\ldots, r \}$, on the other hand, $\bar\sigma_i=0$ and we have
\begin{align*}
 \frac12 \alpha_i = \bar\sigma_i + {\nu_{\{ d_i \}}}\sum_{j=1}^{r}{\bar\sigma_j} - \sigma_i(\bar\M) =  0+ \frac{{\nu_{\{ d_i \}}}}{1+\rho{\nu_{\{ d_i \}}}}\sum_{j=1}^{\rho}{\sigma_j(\bar\M)} - \sigma_i(\bar\M) = - \sigma_i(\bar\M) + \frac{{\nu_{\{ d_i \}}} \rho \kappa_\rho}{1+ \rho {\nu_{\{ d_i \}}}},
\end{align*}
where we used the fact that $$\sum_{i=1}^{r}\bar\sigma_i=1_\rho^\top\bar\sigma_{1:\rho}=\sum_{i=1}^{\rho}\sigma_i(\bar\M) - \frac{{\nu_{\{ d_i \}}}\rho^2 \kappa_\rho }{1+ \rho {\nu_{\{ d_i \}}}}=(1-\frac{{\nu_{\{ d_i \}}}\rho}{1+\rho{\nu_{\{ d_i \}}}})\kappa_\rho=\frac{\rho \kappa_\rho}{1+\rho{\nu_{\{ d_i \}}}}.$$

 By dual feasibility, we conclude that $\sigma_i(\bar\M) \leq \frac{{\nu_{\{ d_i \}}} \rho \kappa_\rho}{1+ \rho {\nu_{\{ d_i \}}}}$ for all $i\in \{\rho+1,\ldots, r\}$, which completes the proof.
\end{proof}

\begin{proof}[Proof of Theorem~\ref{thm:main_1d}]
Consider $\W^*$, a global optimum of problem~\ref{eq:optim}. If all such global optima  can be implemented by equalized networks, then by Theorem~\ref{thm:envelope} it holds that $\Theta(\W^*)=\Theta^{**}(\W^*)=\nu_{\{ d_i \}}\| \W^*\C^\frac12 \|_*^2$. In this case, the lifted problem in Equation~\ref{eq:optim} boils down to the following convex problem
\begin{equation}\label{eq:convex}
\minim{\W\in\R^{d_{k+1}\times d_0}}{\bE[\| \y-\W\x \|^2] + \nu_{\{ d_i \}} \| \W\C^\frac12 \|_*^2}, \quad \rank{\W} \leq \min_{i\in[k+1]}d_i =: r.
\end{equation}
Proposition~\ref{prop:equalizable}, on the other hand, states that any rank-1 network map can be implemented by an equalized network. Therefore, { the key idea of the proof is to make sure that the global optimum of problem~\ref{eq:convex} has rank equal to one.} It suffices to notice that under the assumption $\sigma_1(\bar\M) - \sigma_2(\bar\M) \geq \frac{1}{{\nu_{\{ d_i \}}}}\sigma_2(\bar\M)$, it holds that  $\sigma_1(\bar\M) > \frac{{\nu_{\{ d_i \}}}\sigma_1(\bar\M)}{1+{\nu_{\{ d_i \}}}}$ and $\sigma_j(\bar\M) \leq \frac{{\nu_{\{ d_i \}}}\sigma_1(\bar\M)}{1+{\nu_{\{ d_i \}}}}$ for all $j > 1$. In this case, using Lemma~\ref{lem:global}, the solution $\cS_{\alpha_1}(\bar\M)\C^{-\frac12}$ has rank equal to one, which completes the proof.
\end{proof}